\documentclass{article}

\PassOptionsToPackage{numbers, compress}{natbib}
\usepackage[preprint]{neurips_2026}

\usepackage[utf8]{inputenc} 
\usepackage[T1]{fontenc}    
\usepackage{hyperref}       
\usepackage{url}            
\usepackage{booktabs}       
\usepackage{amsfonts}       
\usepackage{nicefrac}       
\usepackage{microtype}      
\usepackage{xcolor}         

\usepackage{graphicx}
\usepackage{comment}
\usepackage{algorithm}
\usepackage{algorithmic}
\usepackage{amsmath}
\usepackage{amsthm}
\usepackage{amssymb}
\newtheorem{definition}{Definition}

\newtheorem{theorem}{Theorem}
\usepackage{booktabs}
\usepackage{multirow}
\usepackage{wrapfig}  

\definecolor{softgrayblue}{RGB}{60, 130, 130}
\hypersetup{
    colorlinks=true,
    linkcolor=softgrayblue,
    citecolor=softgrayblue,
    urlcolor=black
}

\title{Pooling and Semantic Shift: The Fundamental Challenges in Long Text Embedding and Retrieval}

\author{
Hang Gao$^{1}$ \qquad Wujiang Xu$^{1}$ \qquad Kai Mei$^{1}$ \qquad Dimitris N. Metaxas$^{1}$\\
$^{1}$Rutgers University \\
\texttt{\{h.gao, wujiang.xu, kai.mei\}@rutgers.edu} \quad
\texttt{dnm@cs.rutgers.edu}
}

\begin{document}

\maketitle

\begin{abstract}
  Transformer-based embedding models frequently exhibit geometric pathologies, such as anisotropy and length-induced representation collapse, which can degrade downstream retrieval performance. While prior work often attributes these issues directly to text length or attention mechanisms, we argue that the fundamental drivers are instead the inherent \emph{pooling operations} coupled with internal \emph{semantic shift}. In this paper, we establish a unified theoretical framework proving that contextual pooling intrinsically causes embedding collapse. Specifically, we mathematically prove that pooling semantically diverse sentences inevitably leads to micro-level semantic dilution, and strictly reduces the Mean Pairwise Distance of the vector space, guaranteeing macro-level spatial concentration. Grounded in these geometric insights, we formally define \emph{semantic shift} to capture the natural semantic evolution and dispersion within a text. Through carefully controlled experiments across diverse models and corpora, we disentangle text length from semantic content. We demonstrate that semantic shift is the primary predictor of severe embedding concentration. Crucially, our retrieval evaluations reveal that anisotropy is fundamentally harmful only when induced by strong semantic shifts, reconciling conflicting observations in prior literature and offering a principled explanation for the long-context challenges faced by modern embedding models.
\end{abstract}

\section{Introduction}
Text embeddings have become indispensable for retrieval, question answering, clustering, and a wide range of semantic processing tasks. Classic distributional methods (e.g., Word2Vec \citep{mikolov2013distributed}, GloVe \citep{pennington2014glove}, and fastText \citep{bojanowski2017enriching}) have been largely superseded by Transformer-based Pretrained Language Models (PLM) such as BERT \citep{devlin2019bert} and its variants \citep{liu2019roberta}, as well as GPT-style models \citep{radford2019gpt2}, which produce context-sensitive representations that substantially improve semantic matching.

Despite their empirical success, a growing body of work has revealed that embedding spaces exhibit non-trivial geometric \emph{pathologies}. A widely discussed phenomenon is \emph{anisotropy}, where embeddings concentrate into a narrow cone rather than being uniformly distributed \citep{gao2018representation,ethayarajh2019contextual}. Intuitively, this aggregation of vectors leads to a decrease in the discriminative power of the vectors, which in turn leads to a decrease in the quality of downstream performance. A series of post-processing and normalization techniques have been proposed to mitigate such issues, e.g., removing dominant directions \citep{mu2018allbutthetop,arora2017sif,raunak2019pca},
whitening \citep{su2021whitening,huang2021whiteningbert}, or flow-based transformations \citep{li2020sentence}. However, some studies have shown that these measures do not significantly improve query performance. Therefore, vector concentration/anisotropy remains an interesting open problem. Recent work has identified length-induced embedding collapse, where embeddings of longer texts exhibit reduced variance and become increasingly difficult to distinguish \citep{zhou2025lengthcollapse}. They attribute this effect to the attention mechanism: as input length grows, the attention matrix exhibits a stronger low-pass filtering behavior, accelerating the suppression of high-frequency semantic variations and consequently driving long-text embeddings toward increasingly similar representations.
These observations are important, but incomplete. Although length leads to embedding vector concentration, this is not the fundamental reason. More importantly, vector concentration does not necessarily lead to a decline in downstream query quality. First, let us look at a simple experiment to see why vector concentration does not necessarily lead to a degradation in downstream performance. When we embed the same corpus using different models, the resulting Mean Pairwise Distance (MPD) \citep{ethayarajh2019contextual, ait2023anisotropy}, a common measure of concentration/anisotropy, can vary dramatically.

\begin{wrapfigure}{r}{0.5\textwidth} 
    \centering
    \vspace{-10pt} 
    \includegraphics[width=0.5\textwidth]{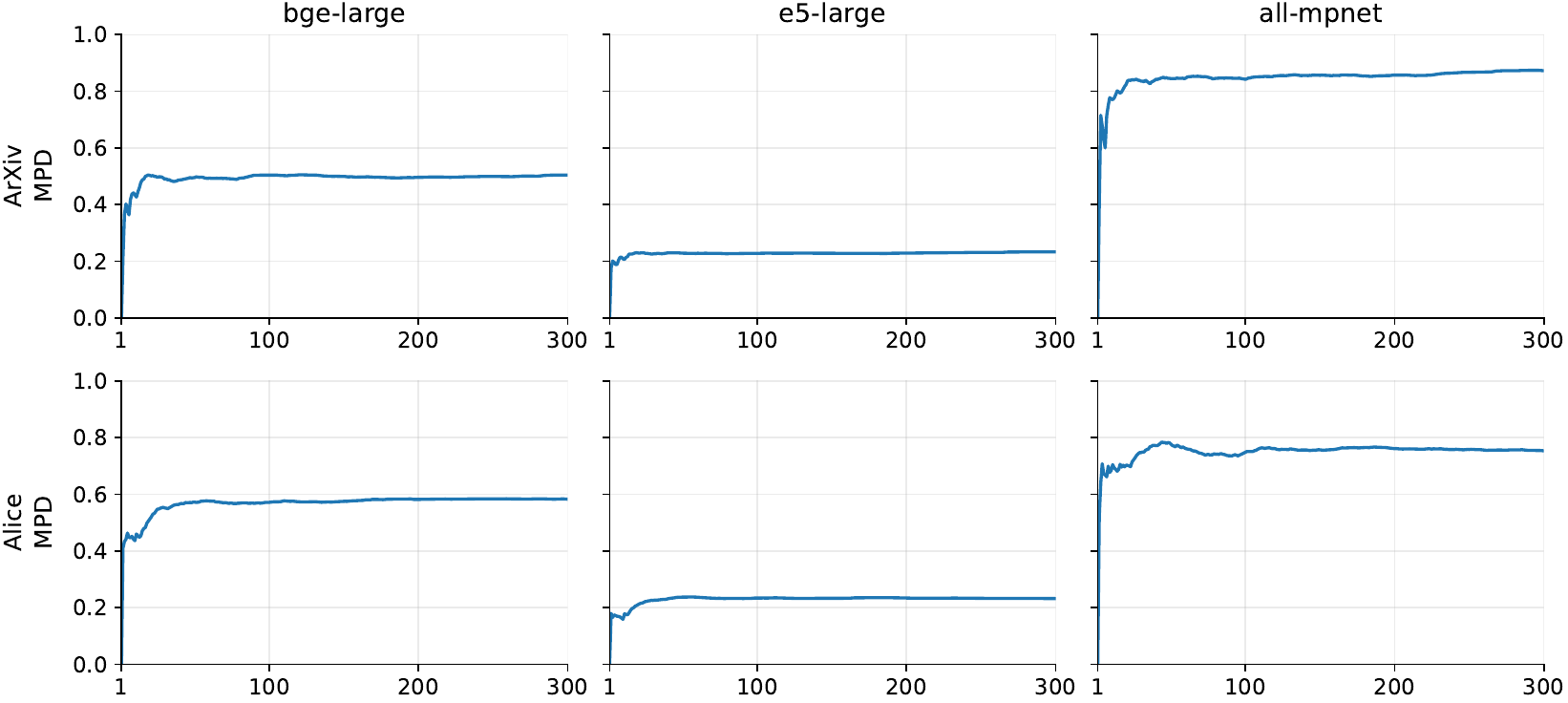} 
    \vspace{-10pt}    \caption{Mean Pairwise Distance (MPD) curves for three embedding models across two corpora. The $x$-axis is the number of sentences; the $y$-axis is MPD.}
    \label{fig:mpd_across_models}
    \vspace{-10pt} 
\end{wrapfigure}

Figure~\ref{fig:mpd_across_models} illustrates how the MPD of sentence embeddings evolves on two corpora, ArXiv~\citep{arxiv_abstracts_hf} and Alice's Adventures in Wonderland~\citep{project_gutenberg,pg_alice}, under several widely used embedding models bge-large~\citep{xiao2024bge}, e5-large~\citep{wang2022e5}, and all-mpnet~\citep{song2020mpnet}. In this experiment, texts are segmented into sentences, all sentences are embedded, and the MPD is computed incrementally over the first 1, 2, \dots, $n$ sentences. As shown in Figure~\ref{fig:mpd_across_models}, the MPD converges to a stable value once $n$ becomes sufficiently large. However, the converged MPD differs drastically across models: bge-large stabilizes around $0.6$, e5-large around $0.2$, and all-mpnet around $0.8$.

Despite these large discrepancies in embedding concentration, these models exhibit broadly comparable performance in practical downstream tasks. This observation makes it difficult to attribute degraded performance solely to anisotropy (i.e., embedding concentration), and it cannot be explained by length-induced embedding collapse either, since all models embed texts of identical length. 

In this paper, we argue that the fundamental factors are \textbf{pooling} and \textbf{semantic shift}. Pooling brings semantic smoothing and dilution, whereas semantic shift in long texts exacerbates smoothing and dilution.
To substantiate this claim, we provide a unified theoretical and empirical framework that identifies pooling and semantic shift as the fundamental driver of embedding pathologies. In summary, our main contributions are:

\textbf{Theoretical Foundation of Semantic Shift and Collapse:} We mathematically prove that pooling mechanisms intrinsically drive representation collapse. Specifically, Theorem~\ref{thm:semantic-dilution} establishes that semantic diversity forces the pooled embedding to shift away from constituent sentences (\emph{semantic dilution}), while Theorem~\ref{thm:pooling-collapse} proves that contextual aggregation strictly reduces the Mean Pairwise Distance (MPD) of the vector space. Grounded in these geometric proofs, we formally define \emph{semantic shift} to jointly capture local semantic evolution and global dispersion within a text.

\textbf{Empirical Disentanglement and Validation:} Through controlled experiments across diverse models and corpora, we disentangle the effects of length and semantic content. We demonstrate that semantic shift—rather than text length—is the primary predictor of severe embedding concentration. Crucially, our retrieval evaluations reveal that anisotropy is fundamentally harmful only when induced by strong semantic shifts, reconciling conflicting observations in prior literature regarding embedding geometry and downstream performance.

\section{Semantic Smoothing in Transformer-Based Embedding Models}
\label{sec:semantic-dilution}

Transformer-based embedding models construct fixed-length text representations by aggregating contextualized token embeddings via mechanisms such as mean pooling in Sentence-BERT~\citep{reimers2019sbert} or \texttt{[CLS]} token attention pooling in BERT~\citep{devlin2019bert}. As we detail in Appendix~\ref{app:pooling}, under a theoretical approximation, the core contextual aggregation in these operations effectively computes a convex combination of its constituent token representations and, by natural extension, in-context span representations:
\(
    c = \sum_{i=1}^k w_i\, e_i, \quad \text{with} \quad \sum_{i=1}^k w_i = 1,
\)
where $e_i$ is the in-context span representation of the $i$-th sentence, $w_i \ge 0$ is its corresponding aggregate weight, and $c$ represents the pooled contextual information. Because this convex combination mechanism fundamentally drives the geometric dynamics of the final embedding, we can conceptually analyze the text embedding behavior by treating a multi-sentence document's contextual representation as being pooled directly from a set of constituent span vectors. This formulation directly sets the stage for our analysis of semantic dilution.

\subsection{Semantic Diversity Forces Semantic Dilution}

Pooling imposes strong geometric constraints: the resulting text embedding must lie in the convex hull of its constituent sentence embeddings. When these sentences are semantically homogeneous, pooling preserves their direction. When they are diverse, pooling forces them into a compromised direction, diluting each sentence's individual meaning.

To make this precise, suppose that a text contains $k$ unit-normalized sentence embeddings $v_1,\dots,v_k \in \mathbb{R}^d$ ($\|v_i\|=1$), and let the pooled embedding and its normalized version be $\mu = \frac{1}{k}\sum_{i=1}^k v_i$ and $\hat\mu = \mu/\|\mu\|$, respectively. We quantify sentence-level semantic diversity by the mean pairwise cosine distance $C_{\mathrm{pair}}$: 

\begin{equation}
C_{\mathrm{pair}} = \frac{1}{k(k-1)} \sum_{1 \le i \neq j \le k} \left( 1 - v_i^\top v_j \right),
\label{eq:Cpair}
\end{equation}

we measure how "unlike" the pooled embedding is relative to the original sentences by $C_{\mathrm{mean}}$:
\begin{equation}
C_{\mathrm{mean}} = \frac{1}{k} \sum_{i=1}^k \left(1 - v_i^\top \hat\mu\right).
\end{equation}

\begin{theorem}[Semantic Dilution]
\label{thm:semantic-dilution}
For any set of unit-normalized sentence embeddings, the discrepancy between the pooled text embedding $\hat\mu$ and the constituent sentences satisfies:
\begin{equation}
    C_{\mathrm{mean}} = 1 - \sqrt{\,1 - \frac{k-1}{k} C_{\mathrm{pair}}\,}.
\end{equation}
Consequently, $C_{\mathrm{mean}}$ is a strictly increasing function of $C_{\mathrm{pair}}$ for all $k \ge 2$.
\end{theorem}

\begin{proof}
We first compute $C_{\mathrm{mean}}$ in terms of $\mu$. Since $\sum_{i=1}^k v_i = k\mu$, we have:
\begin{equation}
    C_{\mathrm{mean}} = \frac{1}{k} \sum_{i=1}^k \left(1 - v_i^\top \hat\mu\right) = 1 - \frac{1}{k\|\mu\|}\sum_{i=1}^k v_i^\top \mu = 1 - \|\mu\|.
    \label{eq:Cmean_mu}
\end{equation}
Next, expanding the squared norm of $\mu$ and substituting the sum $\sum_{i \neq j} v_i^\top v_j = k(k-1)(1 - C_{\mathrm{pair}})$ derived from Equation~\ref{eq:Cpair}, we obtain:
\begin{align}
    \|\mu\|^2 &= \| \frac{1}{k}\sum_{i=1}^k v_i \|^2 = \frac{1}{k^2}\left( k + \sum_{1\le i \neq j\le k} v_i^\top v_j \right) \nonumber \\
              &= \frac{1}{k^2}\Big( k + k(k-1)(1 - C_{\mathrm{pair}}) \Big) = 1 - \frac{k-1}{k} C_{\mathrm{pair}}.
    \label{eq:mu_norm_pair}
\end{align}

Combining Equations~\ref{eq:Cmean_mu} and \ref{eq:mu_norm_pair} yields
\(
    C_{\mathrm{mean}}
    = 1 - \sqrt{\,1 - \frac{k-1}{k}C_{\mathrm{pair}}\,}.
\)
The expression is strictly increasing in $C_{\mathrm{pair}}$ because its derivative
\begin{equation}
    \frac{\partial C_{\mathrm{mean}}}{\partial C_{\mathrm{pair}}}
    = \frac{k-1}{2k}\cdot \frac{1}{\sqrt{1 - \frac{k-1}{k}C_{\mathrm{pair}}}}
\end{equation}
is strictly positive for $k\ge 2$.

\end{proof}

This theorem states that \textbf{the more diverse the sentences that make up a text, the greater the average difference between the overall semantics of the text and the semantics of each individual sentence}.

\textbf{Empirical Validation of Theorem~\ref{thm:semantic-dilution}.}
Theorem~\ref{thm:semantic-dilution} establishes a strict monotonic relationship between sentence-level semantic diversity and text--sentence discrepancy under an idealized pooling assumption. We now empirically verify that this relationship also holds in practice for real Transformer-based embedding models, \textbf{where text embeddings are produced by encoding the concatenated text directly rather than by explicit sentence embedding averaging.}

\begin{wrapfigure}{r}{0.5\textwidth}
    \centering
    \vspace{0pt} 
    \includegraphics[width=0.5\textwidth]{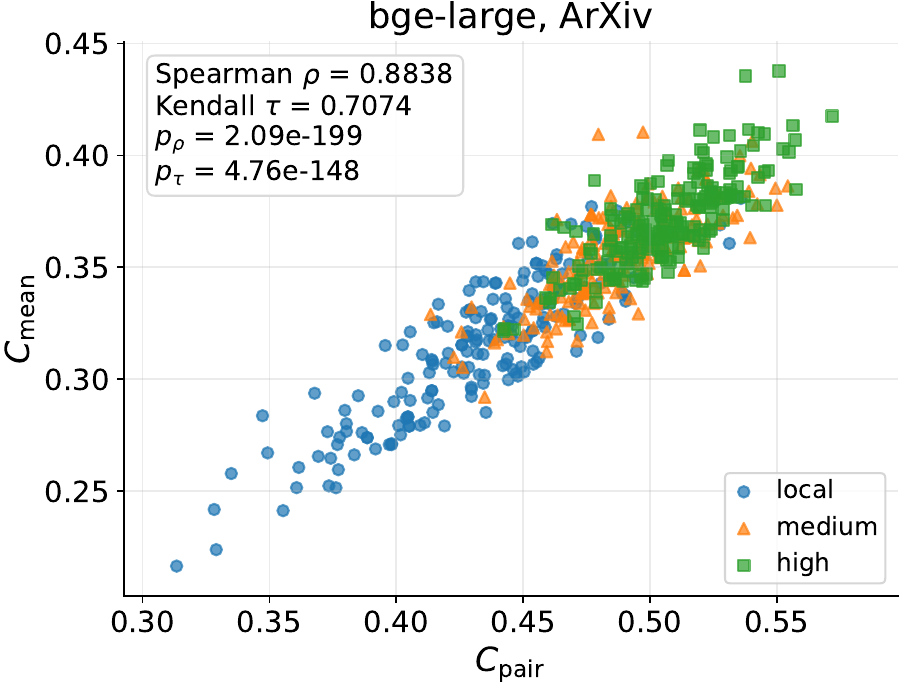}
    \vspace{-10pt}
    \caption{Scatter plot of $C_{\mathrm{mean}}$ vs.\ $C_{\mathrm{pair}}$ on ArXiv using bge-large model.}
    \label{fig:arxiv_cpair_cmean}
    \vspace{-10pt} 
\end{wrapfigure}

Using the ArXiv~\citep{arxiv_abstracts_hf} corpus and the bge-large model~\citep{xiao2024bge}, we construct sentence groups of size $k=10$ under three sampling regimes: local (consecutive sentences), medium (non-adjacent sentences) and high (uniformly random sentences from the corpus), repeating each regime 200 times. In each trial, we select 10 sentences according to the corresponding regime, concatenate them into a single text, encode the text once, and obtain the embeddings of the constituent sentences by encoding each sentence separately. We then compute the mean pairwise cosine distance between sentence embeddings, $C_{\mathrm{pair}}$, and the mean cosine distance between the text embedding and its constituent sentence embeddings, $C_{\mathrm{mean}}$. Additional results across different models and corpora are provided in Appendix~\ref{app:theorem1-empirical}.

\textbf{Correlation between $C_{\mathrm{pair}}$ and $C_{\mathrm{mean}}$.}
Figure~\ref{fig:arxiv_cpair_cmean} reports the scatter plot of $C_{\mathrm{mean}}$ versus $C_{\mathrm{pair}}$ under three sampling regimes. We observe a strong monotonic association:
Spearman's rank correlation is $\rho = 0.8838$ and Kendall's $\tau = 0.7074$ (both highly significant with $p \ll 10^{-100}$). \textbf{This empirical result supports Theorem~\ref{thm:semantic-dilution} in practice: as sentence-level semantic diversity increases (larger $C_{\mathrm{pair}}$), the discrepancy between the concatenated-text embedding and its constituent sentence embeddings also increases (larger $C_{\mathrm{mean}}$).}

\subsection{Representation Collapse and the Impact of Pooling Scale}

Theorem~\ref{thm:semantic-dilution} establishes that the pooled embedding of a semantically diverse text inevitably shifts away from each of its individual constituent sentences. By extending this geometric insight to a macroscopic level, we can derive direct mathematical evidence for representation collapse across a vector space. 

In this section, we analyze a scenario where localized pooling is applied to a normalized embedding space $V = \{v_1, \dots, v_n\}$. We characterize the dispersion of the space using the \textit{Mean Pairwise Distance} (MPD), calculated via dot-product or cosine distance. We show that such pooling not only reduces the MPD but that the severity of this collapse is strictly governed by the pooling scale $k$.

\begin{theorem}[Pooling-Induced Collapse and Scale Dependence]
\label{thm:pooling-collapse}
Given a set of $n$ unit-normalized vectors $V = \{v_1, ..., v_n\}$, let $u_i = \frac{1}{k+1} (v_i + \sum_{x \in S_i} x)$ be the pooled vector where $S_i$ is a set of $k$ vectors ($1 \le k < n-1$) sampled uniformly and independently at random from $V \setminus \{v_i\}$. The Mean Pairwise Distance (MPD) of the set $V$ and the expected MPD of the pooled set $U = \{u_1, \dots, u_n\}$ satisfies:
    
(i) $\text{MPD}(V) > \mathbb{E}[\text{MPD}(U)]$, proving an inevitable collapse. 

(ii) The reduction in MPD, defined as $\Delta\text{MPD} = \text{MPD}(V) - \mathbb{E}[\text{MPD}(U)]$, is strictly monotonically increasing with the pooling scale $k$.
\end{theorem}

\begin{proof}
The original MPD of $V$ is $\text{MPD}(V) = 1 - S_V$, where $S_V = \frac{1}{n(n-1)} \sum_{i \neq j} v_i^\top v_j$ is the mean pairwise similarity (inner product similarity). Let $\mu = \frac{1}{n}\sum_{i=1}^{n} v_i$ be the global mean. The squared norm of the mean relates to $S_V$ as

\begin{equation}
    \|\mu\|^2
    = \left\|\frac{1}{n}\sum_{i=1}^n v_i\right\|^2
    = \frac{1}{n^2}\left(
        n + \sum_{1\le i \neq j\le n} v_i^\top v_j
      \right)
    = \frac{1 + (n-1)S_V}{n}
\end{equation}

Since $S_i$ is sampled uniformly from the remaining $n-1$ vectors, the expected sum in the pooling operation is $\mathbb{E}[\sum_{x \in S_i} x] = \frac{k}{n-1} \sum_{j \neq i} v_j = \frac{k}{n-1} (n\mu - v_i)$. Substituting this into $u_i$ yields:
\begin{equation}
    \mathbb{E}[u_i] = \frac{1}{k+1} \left( v_i + \frac{k}{n-1}(n\mu - v_i) \right) = \alpha(k) v_i + (1-\alpha(k))\mu
\end{equation}
where $\alpha(k) = \frac{n-1-k}{(k+1)(n-1)}$ is the shrinkage factor. For $1 \le k < n-1$, $0 < \alpha(k) < 1$. 

The expected pairwise similarity of $U$ requires computing the expected inner product $\mathbb{E}[u_i]^\top \mathbb{E}[u_j]$. We first expand this term using the expected values of $u_i$ and $u_j$:
\begin{equation}
\begin{aligned}
    \mathbb{E}[u_i]^\top \mathbb{E}[u_j] &= (\alpha v_i + (1-\alpha)\mu)^\top (\alpha v_j + (1-\alpha)\mu) \\
    &= \alpha^2 v_i^\top v_j + \alpha(1-\alpha)(v_i^\top \mu + v_j^\top \mu) + (1-\alpha)^2 \|\mu\|^2
\end{aligned}
    \label{eq:Emu_Emu}
\end{equation}
We then average this over all $i \neq j$ pairs to find $\mathbb{E}[S_U] = \frac{1}{n(n-1)} \sum_{i \neq j} \mathbb{E}[u_i]^\top \mathbb{E}[u_j]$. The first term of Equation~\ref{eq:Emu_Emu} naturally evaluates to $\alpha^2 S_V$. For the cross terms involving $\mu$, we can simplify the double summation by converting it into a single summation, since $\sum_{i=1}^n v_i = n\mu$:
\begin{equation}
    \sum_{1 \le i \neq j \le n} v_i^\top \mu 
    = \sum_{i=1}^n \sum_{\substack{j=1 \\ j \neq i}}^n v_i^\top \mu
    = (n-1) \sum_{i=1}^n v_i^\top \mu 
    = (n-1)(n\mu)^\top \mu 
    = n(n-1)\|\mu\|^2
\end{equation}
By symmetry, $\sum_{i \neq j} v_j^\top \mu$ yields the identical result. Thus, averaging these cross terms over the $n(n-1)$ pairs contributes $2\alpha(1-\alpha)\|\mu\|^2$. The third term of Equation~\ref{eq:Emu_Emu} is a constant scalar and contributes $(1-\alpha)^2 \|\mu\|^2$. Combining these components gives:
\begin{equation}
\begin{aligned}
    \mathbb{E}[S_U] &= \alpha^2 S_V + \left[ 2\alpha(1-\alpha) + (1-\alpha)^2 \right] \|\mu\|^2 \\
    &= \alpha^2 S_V + (1-\alpha^2)\|\mu\|^2
\end{aligned}
\end{equation}

The absolute reduction in dot-product MPD is then:
\begin{equation}
\begin{aligned}
    \Delta\text{MPD} &= \text{MPD}(V) - \mathbb{E}[\text{MPD}(U)] = \mathbb{E}[S_U] - S_V \\
    &= (1-\alpha(k)^2)(\|\mu\|^2 - S_V) = (1-\alpha(k)^2) \frac{1-S_V}{n}
\end{aligned}
\end{equation}
Since $S_V < 1$ and $\alpha < 1$, $\Delta\text{MPD} > 0$, proving the collapse (i). To analyze the impact of $k$, we differentiate $\alpha(k)$ with respect to $k$:
\begin{equation}
    \frac{\partial \alpha}{\partial k} = \frac{-(n-1)(k+1) - (n-1-k)(n-1)}{(k+1)^2(n-1)^2} = -\frac{n}{(k+1)^2(n-1)} < 0
\end{equation}
As $k$ increases, $\alpha(k)$ strictly decreases toward 0. Consequently, the term $(1-\alpha(k)^2)$ strictly increases, meaning $\Delta\text{MPD}$ grows larger with $k$ (ii). In the limit $k \to n-1$ (global pooling), $\alpha \to 0$ and all vectors collapse to the global mean $\mu$.

For cosine distance, the collapse is further accelerated by norm decay. According to Jensen's inequality, $\|u_i\| < 1$. As $k$ increases, $u_i$ is formed by averaging more divergent unit vectors, leading to a smaller $\|u_i\|$. Since the cosine similarity $\frac{u_i^\top u_j}{\|u_i\|\|u_j\|}$ has an increasing numerator and a decreasing denominator, the MPD reduction for cosine distance is strictly greater and grows more rapidly with $k$ than in the dot-product case.
\end{proof}

This theorem provides a principled explanation for the sensitivity of embedding-based systems to the pooling window size. \textbf{In the Transformer structure that employs pooling mechanism, increasing the context or the number of neighbors (increasing $k$) does not merely smooth the signal but mathematically guarantees a deeper collapse of the representation space toward a non-discriminative global mean.}

\textbf{Empirical Validation of Theorem~\ref{thm:pooling-collapse}.} 
To intuitively illustrate the geometric implications of Theorem~\ref{thm:pooling-collapse}, we simulate the pooling-induced collapse in a 3D vector space. We initialize $n=100$ vectors uniformly distributed on the surface of a unit sphere (representing the original embedding space $V$). We then apply the randomized pooling operation on each vector using varying pooling scales $k \in \{2, 4, 8\}$ and plot the resulting spatial distribution of the pooled vectors $U$.

\begin{figure}[htbp]
    \centering
    \vspace{0pt}
    \includegraphics[width=\linewidth]{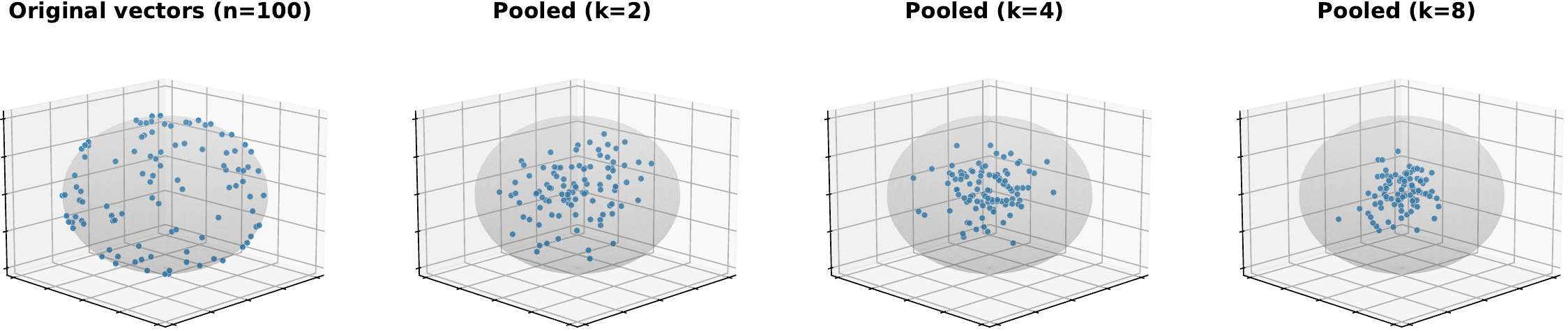}
    \vspace{0pt}
    \caption{Spatial distribution of $n=100$ vectors before (Original) and after randomized pooling with varying scales $k$. As $k$ increases, vectors undergo severe norm decay and angular contraction toward the global mean, visually confirming the strict monotonic MPD reduction proved in  Theorem~\ref{thm:pooling-collapse}.}
    \vspace{-10pt}
    \label{fig:collapse_simulation}
\end{figure}

As shown in Figure \ref{fig:collapse_simulation}, the original vectors are fully dispersed. However, as the pooling scale $k$ increases, the vectors cluster in an increasingly narrow spatial cone, directly reflecting the strict and monotonically increasing reduction in MPD as $k$ grows. This visual simulation corroborates our theoretical proof: contextual aggregation intrinsically acts as a gravitational pull toward the semantic center, and expanding the pooling window ($k$) strictly accelerates this geometric collapse.

\subsection{Implications for Representation Collapse and Anisotropy}

Theorem~\ref{thm:semantic-dilution} and Theorem~\ref{thm:pooling-collapse} jointly provide a geometric explanation for the representational pathologies commonly observed in Transformer-based embeddings, while also clarifying the exact conditions under which pooling harms retrieval.

\textbf{Semantic smoothing and dilution.} 
Theorem~\ref{thm:semantic-dilution} mathematically demonstrates that pooling acts as a smoothing operator. As the semantic diversity among constituent sentences increases, the aggregated embedding inevitably dilutes specific meanings, shifting away from individual sentence representations.

\textbf{Anisotropy as a geometric consequence.} 
Theorem~\ref{thm:pooling-collapse} extends this insight to a macro perspective, proving that pooling strictly reduces the MPD of the vector space. Crucially, this indicates that global anisotropy and representation collapse are not inherent defects of the embedding model, but rather geometric inevitabilities of pooling diverse semantics.

\textbf{Semantic shift as the missing causal factor.}
Synthesizing both theorems reveals a critical insight: pooling itself does not inherently harm retrieval. If a text's constituent sentences are highly similar (low diversity), Theorem~\ref{thm:semantic-dilution} ensures that the pooled vector remains faithful to its original components, and Theorem~\ref{thm:pooling-collapse} implies minimal macroscopic spatial contraction. Severe collapse and retrieval degradation occur only when texts exhibit substantial internal semantic variation. This observation explicitly isolates intra-text variation as the true culprit, motivating our formalization of \textit{semantic shift} in the next section.

\section{Semantic Shift: Formalization and Properties}
\label{sec:semantic-shift}

In natural language, semantics evolve gradually. Adjacent sentences typically share strong local coherence, while sentences farther apart may describe different entities, events, or topics. This structured progression is neither random noise nor model-induced drift; rather, it reflects the intrinsic content-driven evolution of meaning. We refer to this phenomenon as \textbf{semantic shift}. In this section, we formalize semantic shift and argue why it offers a fundamental perspective on embedding pathologies.
A natural attempt to capture semantic evolution is to sum the distances between consecutive sentences.

\begin{definition}[Local Semantic Evolution]
\label{def:naive_semantic_shift}
For sentence embeddings $e_1,\dots,e_k$, the local semantic evolution up to length $k$ is
\(
\mathrm{Local}(k)
= \sum_{i=1}^{k-1} \bigl( 1 - \cos(e_i, e_{i+1}) \bigr), 
\)
with the convention $\mathrm{Local}(1) = 0$.
\end{definition}

This Local Semantic Evolution shows where sentences gradually move away from earlier ones in a coherent direction. However, purely local measures cannot distinguish coherent progression from topic mixing, nor can they capture how far the sentence set spreads in the embedding space. Since semantic dilution (Theorem~\ref{thm:semantic-dilution}) depends on the global diversity of sentence embeddings, a more complete definition must incorporate both local and global information.

We quantify how semantically dispersed a set of $k$ sentences is using their mean pairwise distance.

\begin{definition}[Semantic Dispersion]
\label{def:mpd_k}
Given sentence embeddings $e_1,\dots,e_k$, the \emph{semantic dispersion} is
\(
\mathrm{Disp}(k)
= \frac{2}{k(k-1)} 
  \sum_{1 \le i < j \le k} \bigl(1 - \cos(e_i, e_j)\bigr),
\)
with the convention $\mathrm{Disp}(1) = 0$.
\end{definition}

A larger $\mathrm{Disp}(k)$ indicates that the sentences occupy a wider region in the embedding space.

Semantic shift occurs when the local semantic evolution interacts with the global semantic dispersion. When both are small, the text maintains a stable topic; when both are large, semantics evolve into distinct conceptual regions, creating strong dilution under pooling.

We therefore define semantic shift as the interaction between these two factors.

\begin{definition}[Semantic Shift]
\label{def:semantic-shift}
For a sequence of $k$ sentence embeddings $e_1,\dots,e_k$, the \emph{semantic shift} is defined as
\(
\mathrm{Shift}(k)
= \mathrm{Local}(k) \cdot \mathrm{Disp}(k).
\)
\end{definition}

\section{A New Lens on Representation Collapse and Anisotropy}
\label{sec:length-collapse}
Theoretical results in Sections~\ref{sec:semantic-dilution} and \ref{sec:semantic-shift} suggest that embedding pathologies commonly attributed to text length may, in fact, be driven by pooling and semantic shift. In this section, we design controlled experiments to disentangle these factors and show that pooling and semantic shift are the primary determinants of concentration, anisotropy, and retrieval degradation.

\subsection{How Semantic Shift Drives Length Collapse and Anisotropy}
\label{subsec:shift-vs-concentration}

\paragraph{Experimental setup.}
In the experiments, we used a diverse set of embedding models of different types and scales, including bge-large~\citep{xiao2024bge}, e5-large~\citep{wang2022e5}, all-mpnet~\citep{song2020mpnet}, gte-large~\citep{li2023gte}, and text-embedding models~\citep{openai2024embeddings}, covering open source and closed source systems. On the corpus side, we also evaluated a broad range of text sources, including academic documents, long-form novels, knowledge-focused articles, and encyclopedic materials. Due to space limitations, we present only a subset of the results in the main paper—showcasing selected models and selected corpora. Additional results, along with full details on the models and datasets used, are provided in the Appendix~\ref{sec:appendix_models_corpora}.

For each corpus, we segment the text into sentences to obtain an ordered sequence
\[
S = (s_1, s_2, \dots, s_n).
\]
We then construct longer "sentences" by concatenating sentences in \(S\) according to three patterns, and embed all resulting sequences with a fixed PLM encoder (e.g., bge-large). For each resulting sequence, we measure the embedding concentration using MPD.

\textbf{Repeat concatenation.} Each sentence is repeated multiple times:
\[
\begin{aligned}
S2^{\text{rep}} &= (s_1 s_1,\; s_2 s_2,\; \dots,\; s_n s_n), 
S5^{\text{rep}} = (s_1^5,\; s_2^5,\; \dots,\; s_n^5), \\
S10^{\text{rep}} &= (s_1^{10},\; s_2^{10},\; \dots,\; s_n^{10}),
\end{aligned}
\]
where \(s_i^m\) denotes \(s_i\) repeated \(m\) times. Here, length increases but the underlying semantics of each unit do not change.

\textbf{Sequential concatenation.} Each sentence is concatenated with its immediate successors:
\[
\begin{aligned}
S2^{\text{seq}} &= (s_1 s_2,\; s_2 s_3,\; \dots,\; s_{n-1} s_n,\; s_n), 
S5^{\text{seq}} = (s_1 \dots s_5,\; s_2 \dots s_6,\; \dots,\; s_n), \\
S10^{\text{seq}} &= (s_1 \dots s_{10},\; s_2 \dots s_{11},\; \dots,\; s_n).
\end{aligned}
\]
Here, length increases and semantics evolve smoothly within a local window along the original text.

\textbf{Random concatenation.} Each sentence is concatenated with randomly sampled sentences from the entire corpus:
\[
\begin{aligned}
S2^{\text{rand}} &= (s_1 s_{i_1},\; s_2 s_{i_2},\; \dots,\; s_{n-1} s_{i_{n-1}},\; s_n), 
S5^{\text{rand}} = (s_1 s_{i_1} s_{j_1} s_{k_1} s_{\ell_1},\; \dots,\; s_n), \\
S10^{\text{rand}} &= (s_1 \dots,\; s_2 \dots,\; \dots,\; s_n),
\end{aligned}
\]
where \(s_{i_1}, s_{j_1}, s_{k_1}, s_{\ell_1}, \dots\) are sentences sampled independently from \(S\). Here, both length and semantic heterogeneity increase.

In all three patterns, we embed the resulting sequences and compute the MPD of the embeddings for \(S\), \(S2\), \(S5\), and \(S10\) (where we omit the superscripts in the figure labels for brevity). A lower MPD indicates a stronger embedding concentration, corresponding to the phenomena described by length collapse and anisotropy.

\textbf{Results and analysis.} Figure~\ref{fig:mpd_repeat_seq_rand} summarizes the MPD changes across two corpora (ArXiv~\citep{arxiv_abstracts_hf} and Alice's Adventures in Wonderland~\citep{project_gutenberg,pg_alice}) and different concatenation patterns; the embedding model is bge-large~\citep{xiao2024bge}

For the ArXiv corpus, under repeat concatenation, MPD decreases slowly as we move from \(S\) to \(S10\), while under sequential concatenation, the MPD decreases more rapidly. Under random concatenation, MPD decreases much more sharply. From \(S\) to \(S10\), the drop in MPD is roughly three times larger than in the repeat pattern.

This indicates that pure lengthening (repeat) induces some concentration but not dramatically. In contrast, sequential and random concatenation injects a strong semantic shift within each concatenated unit, causing semantics to be diluted, resulting in a much more severe embedding concentration.

\begin{wrapfigure}{r}{0.5\textwidth}
  \centering
  \vspace{-10pt}
  \includegraphics[width=0.5\textwidth]{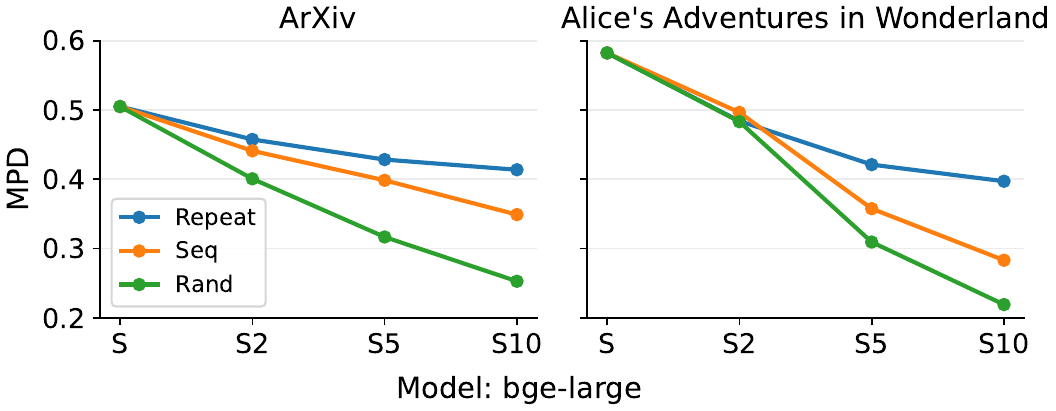}
  \vspace{-10pt}
  \caption{Variation of MPD under different sentence concatenation patterns across two corpora.}
  \label{fig:mpd_repeat_seq_rand}
  \vspace{-10pt}
\end{wrapfigure}

For the Alice's Adventures in Wonderland corpus, we observe a similar overall trend, except that the range of variation in MPD is wider, which is likely due to the different types of corpora.

Across two corpora, the MPD drop from \(S\) to \(S10\) is larger under sequential and random concatenation, again supporting the view that strong internal semantic variation, rather than length alone, is the main driver of severe embedding concentration.

To directly quantify how semantic shift contributes to embedding concentration, we next measure the semantic shift defined in Definition~\ref{def:semantic-shift} under the three concatenation patterns. Specifically, for each corpus and each pattern, we take the $S10$ variant and compute semantic shift at different hop distances: 1-hop, 2-hop, …, 9-hop. This evaluates how semantics evolve as we move further along the concatenated units.

\begin{wrapfigure}{r}{0.5\textwidth}
  \centering
  \vspace{-10pt}
  \includegraphics[width=0.5\textwidth]{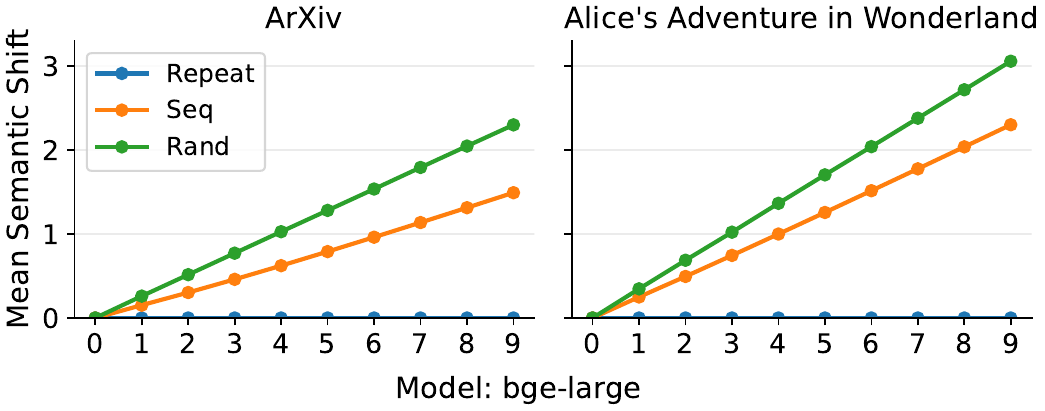}
  \vspace{-10pt}
  \caption{Mean semantic shift increases with hop distance across two corpora under different concatenation modes. The $x$-axis is the hop distance; the $y$-axis is mean semantic shift.}
  \label{fig:semantic_shift}
  \vspace{-10pt}
\end{wrapfigure}

Figure~\ref{fig:semantic_shift} reports the mean semantic shift for the two corpora. In the ArXiv corpus, the random concatenation pattern produces a semantic shift substantially higher than the sequential pattern at all hop distances. This confirms that random mixing injects strong semantic variation even within the same concatenated unit. 

For the Alice's Adventures in Wonderland corpus, the semantic shift exhibited by the random and sequential patterns becomes more similar, reflecting the fact that long narrative texts naturally contain topic transitions and plot developments.

Crucially, when we compare Figure~\ref{fig:semantic_shift} with the MPD results in Figure~\ref{fig:mpd_repeat_seq_rand}, the relationship becomes clear: the degree of embedding concentration aligns almost perfectly with the measured semantic shift. Sequential and random concatenation, which produce a larger semantic shift, also induce significantly stronger MPD reduction.

These results provide quantitative evidence for our central claim.
\textbf{Pooling and semantic shift are the dominant factors driving embedding concentration}, which also motivate our next question. \textbf{When and why does such concentration actually hurt downstream tasks?}

\subsection{Impact on Downstream Retrieval and Revisiting Anisotropy}
\label{subsec:shift-vs-anisotropy}

Anisotropy in embedding spaces, where vectors collapse into a narrow cone, is often reported to be harmful to downstream tasks such as retrieval. However, empirical findings are mixed: some work finds clear negative impacts\citep{gao2018representation, huang2021whiteningbert, gao2021simcse}, while others observe little to no degradation\citep{ait2023anisotropy}.

Building on our theoretical analysis, we hypothesize that anisotropy is harmful primarily when it is induced by a strong semantic shift. To test this, we conduct retrieval experiments on the same corpora and concatenation patterns.

\textbf{Self-overlap as a robustness measure.} For each corpus \(S = (s_1,\dots,s_n)\), we randomly sample 1000 sentences as a query set \(Q\). For each query \(q \in Q\), we perform nearest-neighbor search in the embedding space under the following settings:

\textbf{Baseline}: retrieve top-\(k\) nearest neighbors from the original corpus \(S\).

\textbf{Concatenated variants}: retrieve top-\(k\) neighbors from each of \(S2, S5, S10\) under \emph{repeat}, \emph{sequential}, and \emph{random} patterns.

We treat the top-\(k\) neighbors from \(S\) as a proxy for ground truth, since the set necessarily includes the query itself and its most similar sentences in the original unmodified corpus. For each variant \(S'\) (\(S2, S5, S10\)) and each query \(q\), we compute the \emph{self-overlap@k}:
\[
\text{Overlap@k}(q, S') = \frac{\bigl| \text{Top@k}(q, S) \cap \text{Top@k}(q, S') \bigr|} {k},
\]
and then average over all queries. Higher self-overlap@k means that retrieval on the transformed corpus preserves the same neighbors as the original corpus, indicating weaker damage to retrieval.

\begin{wrapfigure}{r}{0.5\textwidth}
  \centering
  \vspace{-10pt}
  \includegraphics[width=0.5\textwidth]{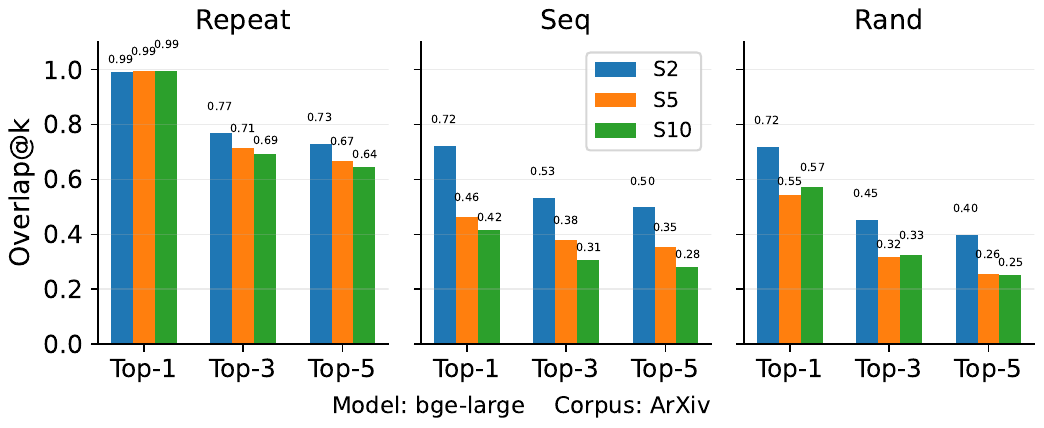}
  \vspace{-10pt}
  \caption{Average self-overlap@k between retrieval results on the original corpus $S$ and its concatenated variants ($S2$, $S5$, $S10$) under repeat, sequential, and random patterns. Higher bars indicate stronger semantic preservation and less retrieval damage.}
  \label{fig:query_anisotropy}
  \vspace{-10pt}
\end{wrapfigure}

\textbf{Results.} Figure~\ref{fig:query_anisotropy} shows the average overlap@k for \(k \in \{1,3,5\}\) across concatenation patterns.

For the \textbf{repeat} pattern: 
Overlap@1 is almost equal to 1.0 for \(S2, S5, S10\), which means that the nearest neighbor is always preserved. 
Overlap@3 and Overlap@5 remain high and stable as length increases.

In contrast, for \textbf{sequential} and \textbf{random} patterns: Overlap@1 drops to about 0.7 and decreases further as we move from \(S2\) to \(S10\). Overlap@3 and Overlap@5 further deteriorate, with random concatenation consistently yielding the lowest overlap.

Across more corpora and different embedding models (see Appendix~\ref{sec:appendix_shift_across_models}, \ref{sec:appendix_overlap_all_models} for full results), the same pattern holds: 
\textbf{Anisotropy driven by length (repeat) leads to mild embedding concentration and has small harm to retrieval.
Anisotropy driven by semantic shift (sequential and random) simultaneously causes strong concentration and substantial retrieval damage.}

In summary, our experiments indicate that anisotropy is not inherently detrimental. It becomes harmful when it arises from semantic shift that blends multiple disparate semantics into single embedding units, thereby collapsing the relative distances on which retrieval relies. This provides a principled explanation of the mixed findings in the literature and highlights semantic shift as the fundamental challenge behind many observed embedding pathologies.

\section{Conclusion}
\label{sec:conclusion}

This paper establishes a unified theoretical framework to explain this phenomenon: Theorem~\ref{thm:semantic-dilution} proves that semantic diversity intrinsically causes micro-level semantic dilution, while Theorem~\ref{thm:pooling-collapse} demonstrates that contextual pooling mathematically guarantees macro-level spatial aggregation and a strictly reduced MPD. Through controlled experiments, we validate these geometric insights, confirming that semantic shift dictates when anisotropy becomes severely harmful to retrieval. 

Beyond diagnosis, semantic shift provides a controllable signal for the design of practical algorithms. As detailed in Appendix~\ref{app:ss_splitter}, we instantiate a Semantic Shift Splitter that leverages this signal for adaptive text segmentation, achieving strong empirical improvements over standard splitters.

Finally, because pooling and attention mechanisms are ubiquitous across Transformer architectures, our theoretical insights extend beyond embedding models. The pooling-induced representation collapse we formally describe offers a compelling geometric explanation for the challenges Large Language Models (LLMs) face with long-context input. Exploring shift-aware aggregation mechanisms to mitigate context dilution in LLMs represents a promising direction for future work.

\bibliographystyle{unsrtnat}
\bibliography{custom}

\clearpage
\onecolumn
\tableofcontents

\appendix
\section*{APPENDIX}
\section{Limitations}
Our analysis interprets Transformer text embeddings through a pooling lens, which cleanly exposes the geometry behind semantic smoothing. Although this abstraction matches common practice (mean/CLS-style pooling) and is empirically supported in our study, it does not attempt to model all fine-grained token interactions across layers. 

We define the semantic shift via cosine-distance-based local evolution and global dispersion. Other reasonable choices (e.g., alternative similarity metrics, different window sizes, or discourse-aware weighting) could be plugged into the same framework and may further refine sensitivity in certain domains. Our goal is to establish a simple and computable measure that is stable across models and corpora, not to claim a unique definition.

For readability, the main text presents representative results on a subset of models/corpora, with additional experiments provided in the appendix. Although the observed trends are consistent across all tested settings (models and corpora), extending coverage to more languages and additional specialized domains would further broaden the empirical picture.

\section{Extended Related Work}
\label{sec:related_work_extended}

\paragraph{Text embeddings and dense representation learning.}
Learning vector representations for text has long been central to information retrieval and semantic matching.
Early approaches focused on static distributional embeddings, such as Word2Vec \citep{mikolov2013distributed}, GloVe \citep{pennington2014glove}, and fastText \citep{bojanowski2017enriching}.
More recently, pretrained language models (PLMs), including BERT \citep{devlin2019bert}, RoBERTa \citep{liu2019roberta}, and GPT-2 \citep{radford2019gpt2}, have enabled context-sensitive representations that significantly improve downstream performance.
For retrieval and similarity-based tasks, dense bi-encoder architectures \citep{karpukhin2020dpr,izacard2021contriever} have become a standard alternative to sparse lexical methods such as BM25 \citep{robertson2009bm25}, while large-scale benchmarks like BEIR \citep{thakur2021beir} and MTEB \citep{muennighoff2023mteb} reveal persistent challenges in robustness and generalization across domains.

\paragraph{Sentence embedding learning and representation geometry.}
A substantial body of work studies how to learn sentence-level embeddings that faithfully capture semantic similarity.
Sentence-BERT \citep{reimers2019sbert} introduced siamese architectures for efficient cosine-based retrieval, which were later enhanced by contrastive learning objectives.
Representative methods include SimCSE \citep{gao2021simcse}, ConSERT \citep{yan2021consert}, Mirror-BERT \citep{liu2021mirrorbert}, PromptBERT \citep{jiang2022promptbert}, DiffCSE \citep{chuang2022diffcse}, and TSDAE \citep{wang2021tsdae}.
More recent embedding families, such as E5 \citep{wang2022e5}, BGE \citep{xiao2024bge}, and INSTRUCTOR \citep{su2023instructor}, aim to produce general-purpose representations aligned with diverse instructions and tasks.

Beyond accuracy, increasing attention has been paid to the \emph{geometry} of embedding spaces.
Prior work shows that contextual embeddings are often highly \emph{anisotropic}, concentrating in a narrow cone rather than being uniformly distributed \citep{ethayarajh2019contextual}.
Several mitigation strategies have been proposed, including removing dominant directions \citep{mu2018allbutthetop,arora2017sif,raunak2019pca}, whitening-based normalization \citep{su2021whitening,huang2021whiteningbert}, and flow-based transformations \citep{li2020sentence}.
Recent analyses caution that global concentration metrics may not reliably predict semantic quality \citep{timkey2021rogue,fusterbaggetto2022anisotropy}.

\paragraph{Long-context representations and length-induced collapse.}
Long texts pose a persistent challenge for embedding-based retrieval. Recent work formalizes this phenomenon as length-induced embedding collapse, attributing it to the low-pass filtering behavior of self-attention, where repeated attention operations progressively suppress high-frequency semantic variations and amplify dominant low-frequency components, causing representations of longer texts to become increasingly similar and less discriminative~\citep{zhou2025lengthcollapse}.
Complementary efforts benchmark long-context embeddings and retrieval~\citep{zhu2024longembed}. Parallel advances in long-context modeling and efficient attention enable substantially longer sequences~\citep{beltagy2020longformer,zaheer2020bigbird,press2022train,dao2022flashattention}. However, length alone does not fully explain retrieval difficulty: texts of identical length can exhibit vastly different embedding behaviors depending on how their semantics evolve internally.

\paragraph{Discourse structure, topic evolution, and semantic shift.}
The evolution of meaning in a document has been extensively studied in linguistics and discourse theory.
Classical frameworks characterize discourse coherence and structure \citep{grosz1986attention,mann1988rst}, while computational models operationalize coherence through entity transitions and local continuity \citep{barzilay2008coherence,guinaudeau2013graph}.
Text segmentation and topic boundary detection formalize discourse evolution, with early unsupervised methods such as TextTiling \citep{hearst1997texttiling} and subsequent statistical approaches \citep{choi2000segmentation,utiyama2001segmentation,galley2003segmentation,malioutov2006mincut,eisenstein2008bayesian}.
Neural models also address segmentation and coherence with supervised or hierarchical architectures \citep{koshorek2018textseg}.

A related literature studies semantic change over time using temporal embeddings and alignment techniques \citep{hamilton2016diachronic,kutuzov2018survey,rudolph2018dynamic,yao2018dynamic}.
Although studies focus on evolution across corpora or time periods, they share a key insight: semantic variation is best understood as a process with measurable rates, rather than as a single static distance.

\section{Token-Level Pooling and Its Sentence-Level Interpretation}
\label{app:pooling}

Transformer encoders construct text embeddings by aggregating contextualized token representations through a fixed pooling mechanism. In this section, we show that the core contextual aggregation in any pooling-based embedding model relies on a convex combination mechanism, which inevitably smooths and dilutes the semantics of a multi-sentence text. This formulation provides a helpful idealized abstraction for Section 2 in the main text, helping us understand anisotropy, length collapse, and their connection to semantic shift.

Let an input text be tokenized as $(x_1,\dots,x_N)$, and let a Transformer encoder produce contextualized token embeddings:
\begin{equation}
    h_1, h_2, \dots, h_N \in \mathbb{R}^d.
\end{equation}
To obtain a fixed-length text embedding, widely used models apply a pooling operator to derive a global sequence representation. Two pooling mechanisms dominate practice. At their core, the contextual aggregation step in both can be conceptually generalized as a weighted sum $c = \sum_{t=1}^N \gamma_t h_t$ under the constraint $\sum_t \gamma_t = 1$:

\paragraph{Mean pooling.}  
Used in Sentence-BERT~\citep{reimers2019sbert}, SimCSE~\citep{gao2021simcse}, E5~\citep{wang2022e5}, BGE~\citep{xiao2024bge}, and many other models. Here, the final representation $z$ is exactly the average of all tokens, meaning the weights are uniformly distributed: $\gamma_t = \frac{1}{N}$.
\begin{equation}
    z = c = \frac{1}{N} \sum_{t=1}^N h_t.
\end{equation}

\paragraph{Attention-weighted pooling.}  
Used in vanilla BERT~\citep{devlin2019bert,clark2019does,ethayarajh2019contextual}, where the \texttt{[CLS]} token gathers global sequence information through self-attention. Modern Transformer architectures compute Multi-Head Attention (MHA) projections and apply position-wise Feed-Forward Networks (FFNs). However, as a theoretical approximation to isolate the routing mechanism, we can abstract the core context-absorbing process. Ignoring residual connections, linear projections of MHA, and non-linearities for this abstraction, the fundamental routing of a single attention head operates as:

\begin{equation}
    c_{\mathrm{CLS}} \approx \sum_{t=1}^N \alpha_t\, h_t,
\end{equation}
where $\gamma_t = \alpha_t \ge 0$ are the attention-derived weights satisfying $\sum_t \alpha_t = 1$. (For this idealized abstraction, we treat the contextualized token embeddings $h_t$ directly as the value vectors, abstracting away the linear value projection). Thus, at the routing level, the mechanism responsible for absorbing context acts conceptually as a convex combination of token representations.

\paragraph{Deriving the Sentence-Level Formulation.}
Since tokens naturally organize into sentences, we can formalize the transition from token-level to sentence-level representations for this core aggregation step $c$. Let the text consist of $k$ sentences, partitioning the index set of $N$ tokens into $k$ disjoint subsets $\mathcal{S}_1, \dots, \mathcal{S}_k$. We can rewrite the generalized aggregation by grouping the tokens by their corresponding sentences:
\begin{equation}
    c = \sum_{i=1}^k \left( \sum_{t \in \mathcal{S}_i} \gamma_t h_t \right).
\end{equation}
Let $w_i = \sum_{t \in \mathcal{S}_i} \gamma_t$ denote the total pooling weight assigned to the $i$-th sentence. Assuming $w_i > 0$, we define the aggregated representation of the $i$-th sentence, $e_i$, as the normalized weighted sum of its constituent token embeddings:
\begin{equation}
    e_i = \frac{1}{w_i} \sum_{t \in \mathcal{S}_i} \gamma_t h_t.
\end{equation}

Crucially, because the token representations $h_t$ are outputs from deep Transformer layers, they are already heavily contextualized across the entire document (i.e., tokens in sentence $i$ have attended to tokens in other sentences). Therefore, $e_i$ technically represents an in-context span representation rather than an independently encoded sentence embedding (such as $v_i$ modeled in Theorem 1).

Substituting $e_i$ back into the aggregation equation yields:
\begin{equation}
    c = \sum_{i=1}^k w_i\, e_i, \quad \text{where} \quad \sum_{i=1}^k w_i = \sum_{i=1}^k \sum_{t \in \mathcal{S}_i} \gamma_t = 1.
\end{equation}

This derivation demonstrates that, under a theoretical approximation, the core contextual aggregation step ($c$) functions as a convex combination of these in-context span representations $e_i$. While practical models mathematically overcomplicate this with MHA concatenations, linear projections, FFNs, and LayerNorms, the underlying geometric dynamics of semantic mixing are driven by this convex routing process. This abstraction justifies the analysis framework in the main text: we can safely treat a multi-sentence document's contextual representation as being pooled directly from a set of constituent span vectors, allowing us to quantify how semantic diversity among these spans forces the final representation into a compromised, diluted direction.

\section{Embedding Models, Corpora, and Preprocessing}
\label{sec:appendix_models_corpora}

This appendix describes the embedding models and corpora used throughout our experiments, together with the unified preprocessing pipeline used to convert each corpus into an ordered sentence sequence $S=(s_1,\dots,s_n)$.

\subsection{Embedding models}
\label{sec:appendix_models}
We consider a diverse set of embedding models that span open-source and proprietary systems, covering different training paradigms and embedding-space geometries. Table~\ref{tab:model_compare} summarizes their key characteristics, including the dimension of the output, the architectural foundation, and the information of the release.

\begin{table*}[htb]
\centering
\small
\setlength{\tabcolsep}{6pt}
\caption{Embedding models used in our experiments. "Dim." denotes the output embedding dimensionality.}
\begin{tabular}{p{1.5cm}|p{1.5cm}|p{3cm}|c|p{1.5cm}|p{1.2cm}|p{1.5cm}}
\hline
\textbf{Model (full name)} & \textbf{Abbreviation} & \textbf{Architecture / key traits} & \textbf{Dim.} & \textbf{Provider} & \textbf{License} & \textbf{References} \\
\hline
bge-large-en-v1.5 & bge-large &
Transformer bi-encoder for dense retrieval; contrastive training with strong general-purpose embedding behavior. &
1024 & BAAI & Open source & \citep{xiao2024bge,chen2024bge} \\
\hline
e5-large-v2 & e5-large &
Transformer bi-encoder trained via weakly-supervised contrastive pretraining; query/passage style prompting is commonly used in the E5 family. &
1024 & Microsoft & Open source & \citep{wang2022e5} \\
\hline
all-mpnet-base-v2 & all-mpnet &
Sentence-Transformers bi-encoder built on MPNet-base; mean pooling for sentence embeddings; widely used strong baseline. &
768 & Sentence-Transformers & Open source & \citep{reimers2019sbert, song2020mpnet} \\
\hline
gte-large & gte-large &
General Text Embeddings (GTE); Transformer encoder optimized for retrieval-style embedding. &
1024 & Alibaba & Open source & \citep{li2023gte} \\
\hline
text-embedding-3-large & text-embedding &
Proprietary API embedding model; high-dimensional embeddings designed for general semantic matching and retrieval. &
3072 & OpenAI & Closed source & \citep{openai2024embeddings} \\
\hline
\end{tabular}
\label{tab:model_compare}
\end{table*}

\subsection{Corpora}
\label{sec:appendix_corpora}
Table~\ref{tab:corpus_compare} summarizes all corpora used in our experiments. These corpora cover heterogeneous discourse regimes (technical abstracts, encyclopedic entries, knowledge essays, and long-form narratives), enabling us to analyze semantic shifts under substantially different topic-evolution patterns.
Unless otherwise noted, we split each document into sentences, preserve the original order, and concatenate all sentences into a single ordered sequence $S=(s_1,\dots,s_n)$. For very large corpora (ArXiv and Wikipedia), we restrict to the first 5{,}000 documents in the dataset in order to control runtime and improve reproducibility.

\begin{table*}[h!]
\centering
\small
\setlength{\tabcolsep}{6pt}
\caption{Corpora used in our experiments. We cover technical, encyclopedic, essay-style, and narrative texts. For large multi-document corpora (ArXiv and Wikipedia), we use the first 5000 documents for efficiency and reproducibility.}
\begin{tabular}{p{2.5cm}|p{1.6cm}|p{5.0cm}|p{3.0cm}}
\hline
\textbf{Corpus (full name)} & \textbf{Abbreviation} & \textbf{Characteristics} & \textbf{References} \\
\hline
ArXiv abstracts (common-pile/arxiv\_abstracts) & ArXiv &
Large-scale scientific paper abstracts. Highly technical, information-dense, and relatively short documents with strong domain-specific terminology.
We use the first 5000 abstracts in dataset order. & \citep{arxiv_abstracts_hf}. \\
\hline
Alice's Adventures in Wonderland (Project Gutenberg) & Alice &
Single long-form narrative novel (fiction). Natural discourse progression with plot-driven topic transitions and stylistic variation.
We treat the entire book as one document and keep the original reading order. & \citep{project_gutenberg,pg_alice}. \\
\hline
Pride and Prejudice (Project Gutenberg) & Pride &
Single long-form narrative novel (fiction). Long-range thematic development and chapter-level transitions.
We treat the entire book as one document and keep the original reading order. & \citep{project_gutenberg,pg_pride}. \\
\hline
MINE essays (kyssen/kg-gen-evaluation-essays) & MINE &
A collection of knowledge-focused essays (multi-document). Each essay is relatively short and expository, often exhibiting clearer local coherence than narratives.
We preserve dataset order and concatenate essays to form $S$. & \citep{mo2025kggen}. \\
\hline
Wikipedia (English) (google/wiki40b) & Wikipedia &
Encyclopedic articles spanning broad topics; expository style with frequent entity/topic changes across documents.
We use the first 5{,}000 documents in dataset order. & \citep{guo2020wiki}. \\
\hline
\end{tabular}
\label{tab:corpus_compare}
\end{table*}

\subsection{Preprocessing and sentence sequence construction}
\label{sec:appendix_preprocess}

In all corpora, we convert the raw text into an ordered sentence sequence $S$ using the same pipeline.

\paragraph{Text cleaning.}
Given a raw text string, we apply a lightweight cleaning function that:
(i) strips noisy characters at both ends while deliberately preserving sentence-final punctuation to avoid breaking sentence boundary detection; and
(ii) merges repeated whitespace into a single space.

\paragraph{Sentence segmentation.}
We split each document into sentences via nltk.tokenize.sent\_tokenize, which is based on the Punkt sentence tokenizer \citep{kiss2006punkt}. After splitting, empty sentences are removed, and each sentence is re-cleaned. This yields a list of sentences for each document.

\paragraph{Preserving order and forming $S$.}
For each corpus, we preserve the original document order, and, within each document, preserve the original sentence order. We then concatenate all sentence lists into one global ordered sequence
$S=(s_1,\dots,s_n)$.
For corpora that naturally consist of many documents (ArXiv, Wikipedia, MINE), this produces a long sequence whose local neighborhoods reflect within-document coherence, while global transitions reflect the dataset's document ordering.

\section{Comparing Mean Pairwise Distance (MPD) Across Corpora and Embedding Models}
\label{sec:appendix-mpd-table}

\begin{table*}[htb]
\centering
\small
\renewcommand{\arraystretch}{1.15}
\caption{Mean pairwise cosine distance (MPD) of sentence embeddings across corpora and embedding models. Larger MPD indicates more dispersed sentence embeddings; smaller MPD indicates stronger concentration. The last column reports the average MPD across models for each corpus, and the last row reports averages across corpora for each model.}
\begin{tabular}{lcccccc}
\toprule
\textbf{Corpus} &
\textbf{bge-large} &
\textbf{e5-large} &
\textbf{all-mpnet} &
\textbf{gte-large} &
\textbf{text-embedding} &
\textbf{Avg.} \\
\midrule
ArXiv                        & 0.505 & 0.231$\downarrow$ & 0.865$\uparrow$ & 0.249 & 0.789 & 0.528 \\
Alice's Adventures           & 0.582 & 0.232 & 0.747$\uparrow$ & 0.202$\downarrow$ & 0.697 & 0.492 \\
Pride and Prejudice          & 0.577 & 0.240 & 0.774$\uparrow$ & 0.209$\downarrow$ & 0.718 & 0.504 \\
MINE                         & 0.575 & 0.253$\downarrow$ & 0.903$\uparrow$ & 0.256 & 0.882 & 0.574 \\
Wikipedia                    & 0.641 & 0.293 & 0.897 & 0.261$\downarrow$ & 0.900$\uparrow$ & 0.598 \\
\midrule
Avg. (across corpora)        & 0.576 & 0.250 & 0.837 & 0.235 & 0.797 & 0.539 \\
\bottomrule
\end{tabular}
\label{tab:mpd-corpus-model}
\end{table*}

This section complements the simplified illustration in
Figure~\ref{fig:mpd_across_models} (main paper) by reporting the
corpus-level MPD statistics for all models and corpora used throughout
the paper, and by providing a more careful interpretation of what MPD does---and
does not---reveal about embedding geometry and downstream retrieval.

\subsection{Metric and computation protocol}

Given a corpus-specific ordered sentence sequence $S=(s_1,\dots,s_n)$ constructed by the preprocessing pipeline in Section~\ref{sec:appendix_preprocess}, we embed each sentence $s_i$ using an embedding model (Section~\ref{sec:appendix_models}) to obtain unit-normalized sentence embeddings $\{e_i\}_{i=1}^n$.
We then compute the mean pairwise cosine distance (MPD):
\begin{equation}
\label{eq:mpd_global}
\mathrm{MPD}(S)
=\frac{2}{n(n-1)}
\sum_{1\le i<j\le n}\bigl(1-\cos(e_i,e_j)\bigr),
\end{equation}
where a smaller MPD indicates a more concentrated (more anisotropic) embedding distribution, while a larger MPD indicates more dispersed sentence embeddings.

The main paper (Figure~\ref{fig:mpd_across_models}) plots an incremental version of this statistic by computing MPD over the first $1,2,\dots,n$ sentences. The plateau observed there motivates a practical summary statistic: the converged MPD value when $n$ is sufficiently large.
Table~\ref{tab:mpd-corpus-model} reports this corpus-level MPD computed on all sentences in each corpus after preprocessing. For large multi-document corpora (ArXiv and Wikipedia), we follow Section~\ref{sec:appendix_corpora} and restrict to the first 5000 documents for efficiency and reproducibility.

\subsection{Results: model dependence vs.\ corpus dependence}

Table~\ref{tab:mpd-corpus-model} shows that MPD varies substantially across both models and corpora, but qualitatively different ways.

\paragraph{Model dependence dominates the absolute MPD scale.}
Keeping the corpus fixed, different embedding models can yield dramatically different MPD values. For example, on ArXiv, MPD ranges from $0.231$ (e5-large) to
$0.865$ (all-mpnet), a gap of $\approx0.634$. Similar gaps appear on Wikipedia (from $0.261$ to $0.900$, gap $\approx0.639$) and MINE (from $0.253$ to $0.903$, gap $\approx0.650$). This agrees with Figure~\ref{fig:mpd_across_models}: even when the MPD curves stabilize as $n$ increases, their converged levels are strongly
model-dependent.

A consistent ranking also emerges in the averaged row of Table~\ref{tab:mpd-corpus-model}: gte-large and e5-large tend to produce the most concentrated sentence embeddings (lowest MPD), bge-large is intermediate, while text-embedding and all-mpnet are substantially more dispersed (higher MPD).

\paragraph{Corpus dependence reflects discourse regime, but with smaller range.}
Keeping the model fixed, MPD still varies across corpora, indicating that sentence-level semantic diversity differs by domain. However, the range within the model is typically much smaller than the range across the model. For example, under bge-large, MPD ranges from $0.505$ (ArXiv) to $0.641$ (Wikipedia), a spread of $0.136$; under e5-large, the spread is $0.062$ (from $0.231$ to $0.293$); under gte-large, the spread is $0.059$ (from $0.202$ to $0.261$).
This pattern suggests that while corpus semantics shape dispersion, the global geometry induced by the embedding model largely determines the overall MPD scale.

At the corpus level, the last column of Table~\ref{tab:mpd-corpus-model} indicates that Wikipedia and MINE have a higher average MPD than the two novels, which is consistent with their broader topical coverage and frequent changes between documents. In contrast, long-form narratives (Alice, Pride) tend to maintain stronger global continuity and recurring entities/themes, which typically reduces global dispersion.

\subsection{Discussion: what MPD can (and cannot) explain}

\paragraph{MPD is a geometry descriptor, not a performance predictor.}
MPD (and related anisotropy/concentration measures) summarizes the global spread of embeddings, but it does not directly determine retrieval quality. This is consistent with the paradox highlighted in the main paper: models with very different MPD (e.g., e5-large vs. all-mpnet) can still achieve broadly comparable performance on practical downstream tasks. In other words, the absolute concentration level alone is insufficient to explain when embedding-based retrieval becomes difficult.

\paragraph{Why can different models yield drastically different MPD?}
The strong model dependence of MPD suggests that it is not merely a property of the corpus. Different training objectives, data mixtures, embedding dimensions, pooling implementations, and normalization conventions can induce different global angular distributions (i.e. different degrees of anisotropy) even on identical inputs. Therefore, comparing MPD values across models mainly reveals differences in embedding-space geometry, not necessarily differences in semantic fidelity.

\paragraph{Implication for our paper.}
Taken together, Table~\ref{tab:mpd-corpus-model} and Figure~\ref{fig:mpd_across_models} motivate the central question of this paper: if global concentration statistics can vary widely across models and yet do not consistently predict downstream behavior, what content-driven factor explains when embeddings become less discriminative? This motivates our semantic shift perspective in the main paper: instead of treating concentration as the root cause, \textbf{we examine how structured semantic evolution within text (semantic shift) interacts with pooling/smoothing mechanisms to produce collapse and retrieval degradation.}

Table~\ref{tab:mpd-corpus-model} should be read as a diagnostic snapshot of the embedding geometry induced by each model in each discourse regime. Its main message is not that "low MPD is bad" or "high MPD is good", but that MPD is strongly model-dependent and therefore cannot by itself serve as a universal explanation of retrieval difficulty. This observation sets the stage for the controlled semantic-shift experiments analyzed in subsequent sections.

\section{Further Analysis of Transformer-Based Embedding Models and Extended Experiments on Theorem~\ref{thm:semantic-dilution}}
\label{app:theorem1-empirical}

This appendix extends the empirical validation of Theorem~\ref{thm:semantic-dilution} to five embedding models and five corpora (summarized in Section~\ref{sec:appendix_models_corpora}).
Our goal is to test the central claim under a realistic encoding pipeline: \textbf{even when a multi-sentence text is encoded directly by a Transformer
encoder (instead of being explicitly averaged over sentence embeddings), sentence-level semantic diversity still monotonically increases the discrepancy
between the text embedding and its constituent sentence embeddings.}

\begin{figure*}[htb!]
  \centering
  \includegraphics[width=\linewidth]{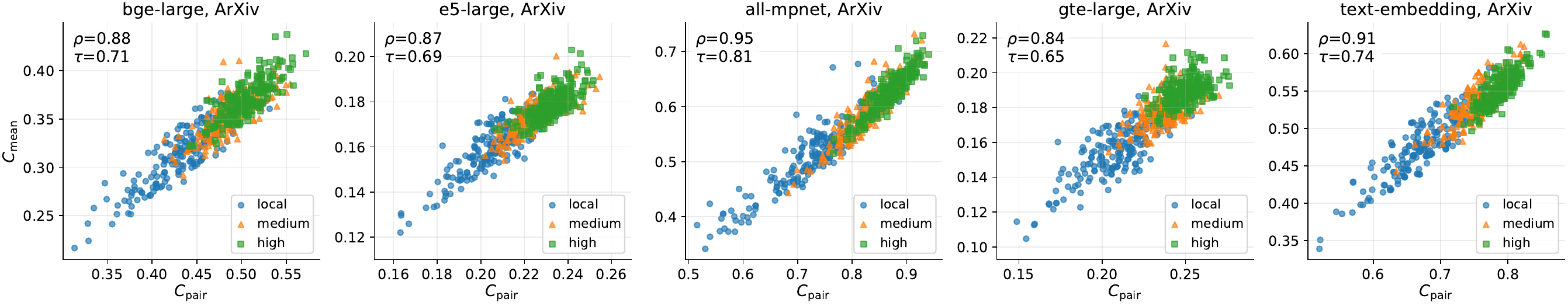}
  \includegraphics[width=\linewidth]{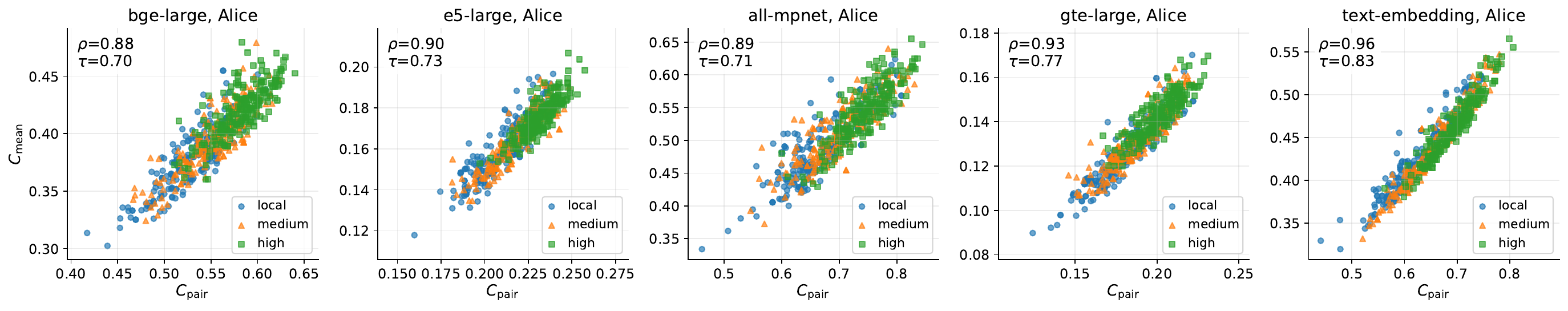}
  \includegraphics[width=\linewidth]{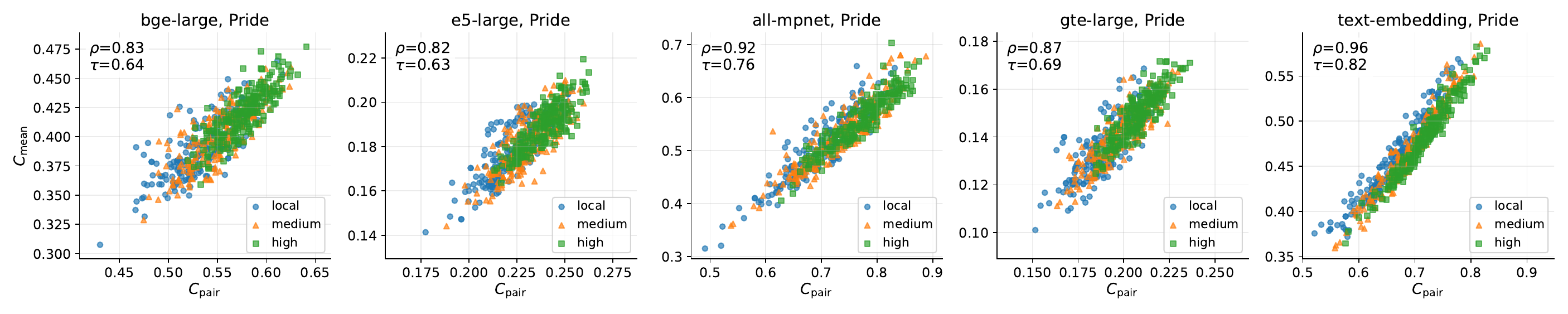}
  \includegraphics[width=\linewidth]{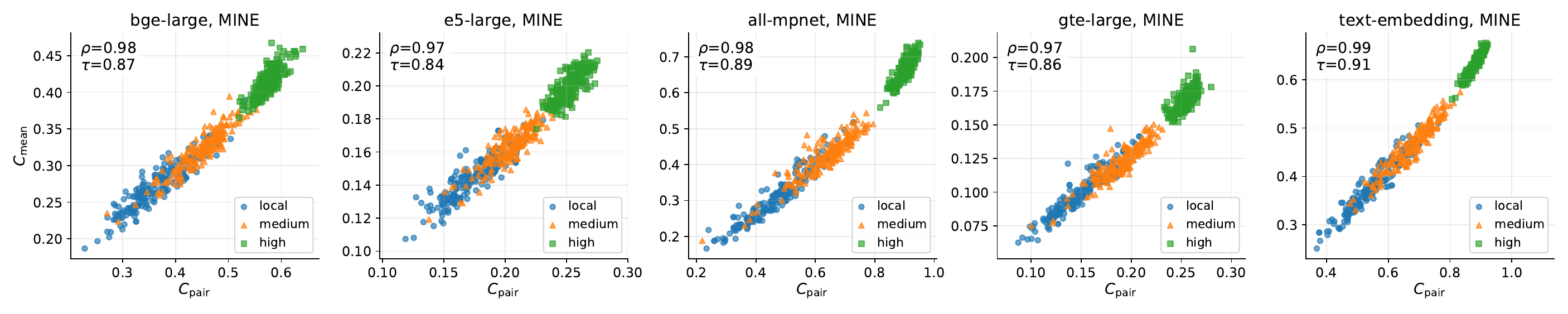}
  \includegraphics[width=\linewidth]{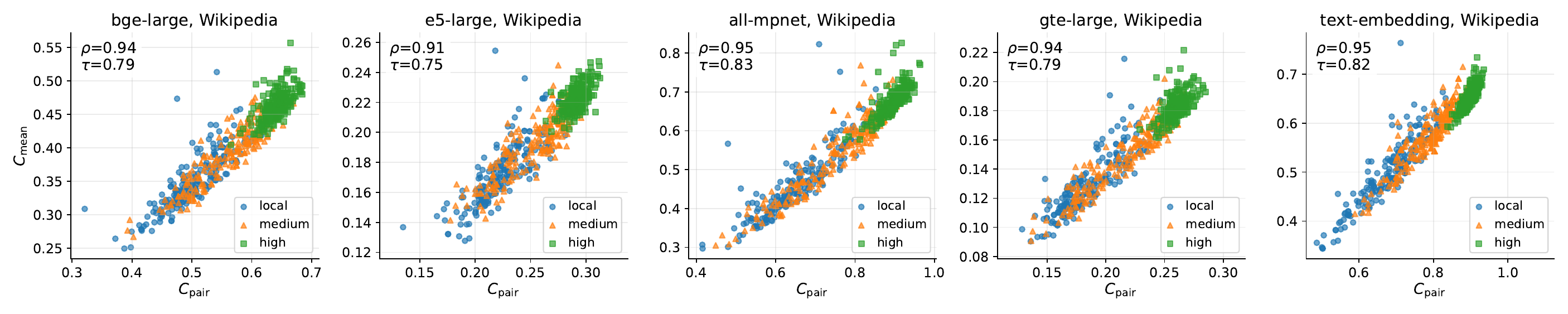}
  \caption{Empirical verification of Theorem~\ref{thm:semantic-dilution} across five corpora and five embedding models. Each subplot shows $C_{\mathrm{mean}}$ (text--sentence discrepancy) versus $C_{\mathrm{pair}}$ (sentence-level semantic diversity) under three controlled diversity regimes: \textit{local}, \textit{medium}, and \textit{high}. Rank correlations (Spearman's $\rho$, Kendall's $\tau$) are reported in each subplot.}
  \label{fig:models_corpora_Cmean_Cpair}
\end{figure*}

\begin{table}[htb]
\centering
\small
\setlength{\tabcolsep}{3pt}
\caption{Extended empirical validation of Theorem~\ref{thm:semantic-dilution} across five corpora and five embedding models. We report Spearman's $\rho$ and Kendall's $\tau$ between $C_{\mathrm{pair}}$ and $C_{\mathrm{mean}}$. High values across all settings indicate a robust monotonic relationship.}
\begin{tabular}{lccccc}
\toprule
\textbf{Corpus} &
\textbf{bge-large} &
\textbf{e5-large} &
\textbf{all-mpnet} &
\textbf{gte-large} &
\textbf{text-embedding} \\
\midrule
ArXiv &
$\rho{=}0.88,\ \tau{=}0.71$ &
$\rho{=}0.87,\ \tau{=}0.69$ &
$\rho{=}0.95,\ \tau{=}0.81$ &
$\rho{=}0.84,\ \tau{=}0.65$ &
$\rho{=}0.91,\ \tau{=}0.74$ \\
Alice &
$\rho{=}0.88,\ \tau{=}0.70$ &
$\rho{=}0.90,\ \tau{=}0.73$ &
$\rho{=}0.89,\ \tau{=}0.71$ &
$\rho{=}0.93,\ \tau{=}0.77$ &
$\rho{=}0.96,\ \tau{=}0.83$ \\
Pride &
$\rho{=}0.83,\ \tau{=}0.64$ &
$\rho{=}0.82,\ \tau{=}0.63$ &
$\rho{=}0.92,\ \tau{=}0.76$ &
$\rho{=}0.87,\ \tau{=}0.69$ &
$\rho{=}0.96,\ \tau{=}0.82$ \\
MINE &
$\rho{=}0.98,\ \tau{=}0.87$ &
$\rho{=}0.97,\ \tau{=}0.84$ &
$\rho{=}0.98,\ \tau{=}0.89$ &
$\rho{=}0.97,\ \tau{=}0.86$ &
$\rho{=}0.99,\ \tau{=}0.91$ \\
Wikipedia &
$\rho{=}0.94,\ \tau{=}0.79$ &
$\rho{=}0.91,\ \tau{=}0.75$ &
$\rho{=}0.95,\ \tau{=}0.83$ &
$\rho{=}0.94,\ \tau{=}0.79$ &
$\rho{=}0.95,\ \tau{=}0.82$ \\
\bottomrule
\end{tabular}
\label{tab:theorem1_extended_corr}
\end{table}

\subsection{Protocol: controlling sentence-level semantic diversity}

We follow the unified preprocessing pipeline in Section~\ref{sec:appendix_models_corpora} to convert each corpus into an ordered sentence sequence. We then fix the group size to $k{=}10$ and construct sentence groups under three controlled diversity regimes: \textit{local} (consecutive sentences within a document), \textit{medium} (non-adjacent sentences within a document) and \textit{high} (sentences sampled uniformly from the corpus).
This design varies sentence-level semantic diversity while holding $k$ fixed, allowing a direct test of the monotonicity predicted by Theorem~\ref{thm:semantic-dilution} beyond idealized pooling.

\subsection{Metrics and evaluation}

For each sampled group $(s_1,\dots,s_k)$, we compute sentence embeddings $(e_1,\dots,e_k)$ by encoding each sentence separately, and compute a text embedding $z$ by concatenating the $k$ sentences (with standard separators) and encoding the resulting multi-sentence text once. We then measure:
\[
\begin{aligned}
&C_{\mathrm{pair}} = \frac{2}{k(k-1)}\sum_{i<j}\left(1-\cos(e_i,e_j)\right),\\
&C_{\mathrm{mean}} = \frac{1}{k}\sum_{i=1}^k \left(1-\cos(e_i,z)\right).
\end{aligned}
\]
$C_{\mathrm{pair}}$ quantifies sentence-level semantic diversity, while $C_{\mathrm{mean}}$ quantifies how much the encoded text representation deviates from its constituent sentences (semantic dilution).

To quantify monotonic dependence without assuming linearity, we report Spearman's rank correlation $\rho$ and Kendall's $\tau$ between $C_{\mathrm{pair}}$ and $C_{\mathrm{mean}}$ for each corpus--model pair.

\subsection{Results: Strong Cross-Model and Cross-Corpus Monotonicity}
Figure~\ref{fig:models_corpora_Cmean_Cpair} reports scatter plots of $C_{\mathrm{mean}}$ versus $C_{\mathrm{pair}}$ for each corpus, with points stratified by the three diversity regimes. Across all corpora and embedding models, we observe a clear monotonic trend: higher sentence-level semantic diversity ($C_{\mathrm{pair}}$) consistently yields larger text--sentence discrepancy ($C_{\mathrm{mean}}$), matching the qualitative behavior predicted by Theorem~\ref{thm:semantic-dilution}.

To quantify monotonic dependence without assuming linearity, we compute Spearman's $\rho$ and Kendall's $\tau$ for each corpus-model pair. The results are summarized in Table~\ref{tab:theorem1_extended_corr}. Correlations are uniformly high: \emph{Spearman $\rho$ ranges from $0.82$ to $0.99$ and Kendall $\tau$ ranges from $0.63$ to $0.91$ across all settings.} Notably, the knowledge-oriented MINE corpus exhibits near-saturated correlations across all models (e.g., $\rho \ge 0.97$), while long-form narratives (Alice / Pride) remain strongly monotonic but slightly noisier—consistent with the fact that narrative texts contain richer discourse phenomena (e.g., gradual topic drift, character/event re-entrance) that can introduce additional variability in embedding behavior.

\paragraph{Why scale differences do not affect the conclusion.}
We emphasize that the absolute magnitudes of $C_{\mathrm{pair}}$ and $C_{\mathrm{mean}}$ can vary substantially across models due to differences in embedding-space geometry (e.g., global angular concentration, normalization conventions, and training objectives). This is precisely why we report rank-based statistics (Spearman/Kendall): they are invariant
to monotone rescaling and directly test the theoretical prediction of monotonicity. Therefore, model-dependent scale differences do not change the conclusion that semantic diversity reliably drives semantic dilution under practical Transformer encoders.

In general, the extended results in Figure~\ref{fig:models_corpora_Cmean_Cpair} and Table~\ref{tab:theorem1_extended_corr} confirm that Theorem~\ref{thm:semantic-dilution} captures a robust property of real embedding models and diverse corpora, rather than a peculiarity of a specific architecture, dataset, or idealized pooling assumption. Sentence-level semantic diversity consistently induces larger text--sentence discrepancy even under direct encoding of concatenated text, providing a solid empirical foundation for the semantic shift perspective developed in the main paper.

\section{Additional Results: Semantic Shift vs. Length Collapse Across Embedding Models}
\label{sec:appendix_shift_across_models}

This appendix complements Sec.~\ref{subsec:shift-vs-concentration} by reporting results on two additional embedding models (e5-large and all-mpnet) beyond the main-embedding model (bge-large). Across all three models and both corpora (ArXiv and Alice's Adventures in Wonderland), we observe highly consistent qualitative patterns: (i) embedding concentration (measured by MPD) strengthens as the constructed text units become longer, but the magnitude of this effect depends primarily on how strongly semantics are mixed within each unit; and (ii) our semantic-shift metric measured on the same constructed units tracks the severity of MPD reduction almost monotonically. These results support the claim that semantic shift is a model-robust explanatory variable for when length collapse and anisotropy become severe.

\begin{wrapfigure}{r}{0.5\textwidth}
  \centering
  \vspace{-15pt}
  \includegraphics[width=0.5\textwidth]{bge_large_mpd.pdf}
  \includegraphics[width=0.5\textwidth]{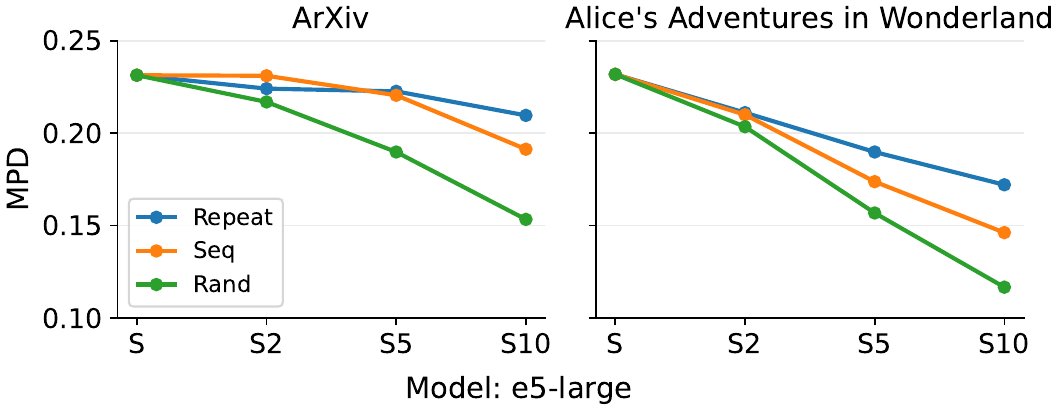}
  \includegraphics[width=0.5\textwidth]{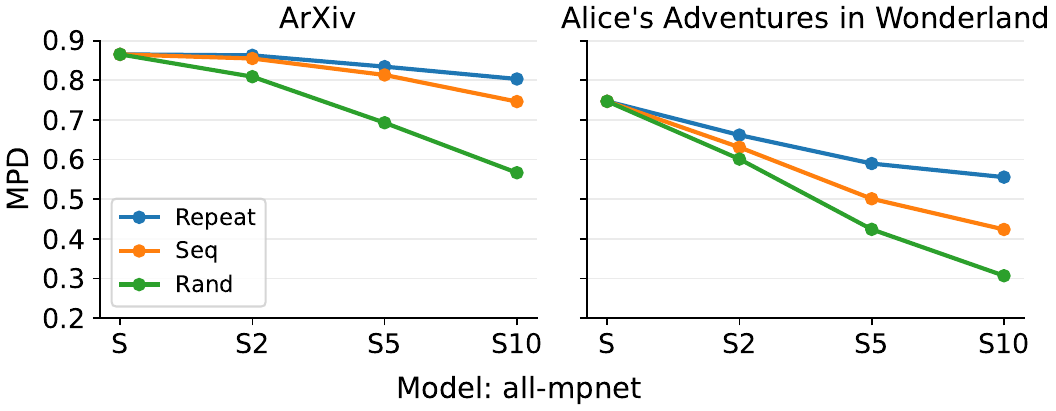}
  \vspace{-10pt}
  \caption{Embedding concentration measured by MPD under different concatenation patterns on ArXiv and \textit{Alice's Adventures in Wonderland}. Lower MPD indicates stronger concentration (i.e., more severe length collapse / anisotropy).}
  \label{fig:models_corpora_mpd}
  \vspace{-20pt}
\end{wrapfigure}

\paragraph{Figures.}
For each model, we report (1) MPD under the three concatenation patterns (repeat, sequential, random) for $S$, $S2$, $S5$, and $S10$, and (2) mean semantic shift measured on the $S10$ variant across hop distances $1\ldots 9$.
Figures~\ref{fig:models_corpora_mpd} and~\ref{fig:models_corpora_shift} summarize the complete set of results.

\paragraph{(1) MPD results are consistent across models.}
Across all models, MPD decreases from $S \rightarrow S10$ in all concatenation patterns (Fig.~\ref{fig:models_corpora_mpd}), confirming that lengthening tends to increase concentration. However, the rate and extent of MPD reduction depend strongly on how semantics are composed within each constructed unit:
\emph{repeat} produces the mildest MPD drop, \emph{sequential} produces a noticeably larger drop, and \emph{random} produces the sharpest decline, indicating the strongest concentration. This ordering (Repeat $<$ Seq $<$ Rand in collapse severity) holds on both corpora for all models (Figs.~\ref{fig:models_corpora_mpd}).

In addition, the absolute MPD level is clearly model-dependent. e5-large operates in a substantially more concentrated regime overall (lower MPD throughout), while all-mpnet is the least concentrated (higher MPD), and bge-large lies in between. This reproduces a familiar empirical fact: different embedding families can exhibit markedly different global geometry (e.g., anisotropy), even when their downstream performance is broadly comparable.
However, critically, these baseline differences do not change the central pattern: \emph{semantic diversity consistently amplifies concentration far more than pure lengthening via repetition}.

\paragraph{(2) Semantic shift curves show the same ranking across models.}
Figure~\ref{fig:models_corpora_shift} reports the mean semantic shift in the variant $S10$ across hop distances. Three robust patterns emerge across all models and both corpora:
(i) Repeat yields zero shift across hop distances, as expected because the constructed units preserve the same sentence semantics and only increase length;
(ii) Sequential and random shifts increase with hop distance, indicating that semantic divergence accumulates progressively as we move farther along the sequence; and
(iii) Random tends to exceed sequential most clearly on ArXiv, reflecting stronger semantic heterogeneity induced by global mixing
(Figs.~\ref{fig:models_corpora_shift}).

For Alice's Adventures in Wonderland, the gap between random and sequential becomes smaller than that on ArXiv across all models (Fig.~\ref{fig:models_corpora_shift}). This is consistent with the narrative nature of the corpus: even local sequential windows naturally include topic transitions and plot development, so sequential concatenation already induces non-trivial within-unit semantic evolution, partially closing the gap to random mixing.

\paragraph{(3) Semantic shift explains MPD reduction better than length alone.}
Comparing Figs.~\ref{fig:models_corpora_mpd} and~\ref{fig:models_corpora_shift} reveals a consistent alignment:
patterns with larger measured semantic shift (Seq/Rand) also produce stronger MPD reductions, while repeat concatenation produces zero semantic shift and only mild concentration. Importantly, this alignment persists across:
\textbf{(a)} embedding models with very different baseline geometry (overall MPD levels), and
\textbf{(b)} corpus types with different discourse properties (technical articles vs. long-form narrative).
Therefore, the expanded results strengthen the main-text conclusion:
\emph{the severity of length collapse / anisotropy is primarily controlled by the strength of within-unit semantic shift, rather than by length itself.}

\paragraph{Cross-model invariance.}
A particularly salient observation is that the e5-large model is globally more anisotropic (lower MPD across all settings), which could prima facie suggest that "anisotropy alone" should predict retrieval difficulty.
However, across bge-large, e5-large, and all-mpnet, we consistently observe the same monotonic relationship:
\textbf{larger semantic shift $\Rightarrow$ larger MPD reduction (stronger collapse)}.
In other words, even when a model starts from a more concentrated geometry, \emph{semantic shift still governs how rapidly embeddings further collapse as we inject semantic heterogeneity}. This invariance indicates that semantic shift is not a model-specific artifact; rather, it functions as a stable explanatory variable that generalizes across embedding families with different training objectives and baseline anisotropy.

\begin{wrapfigure}{r}{0.5\textwidth}
  \centering
  \vspace{-15pt}
  \includegraphics[width=0.5\textwidth]{bge_large_shift.pdf}
  \includegraphics[width=0.5\textwidth]{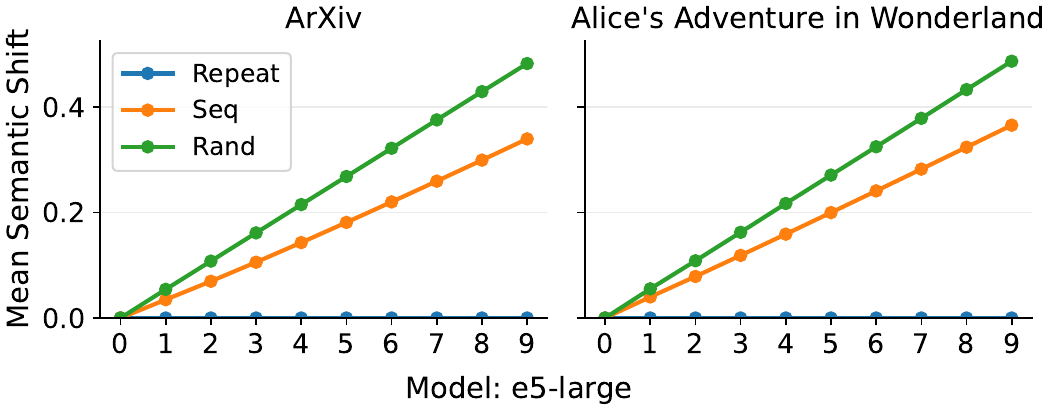}
  \includegraphics[width=0.5\textwidth]{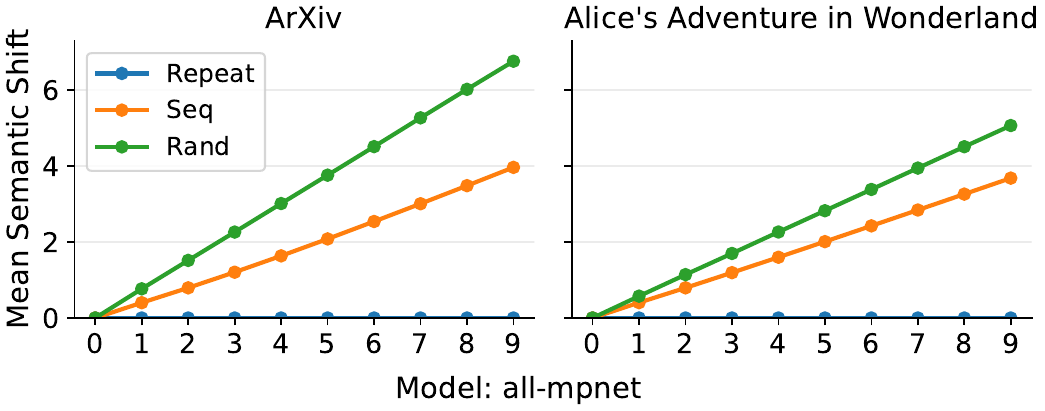}
  \vspace{-10pt}
  \caption{Mean semantic shift on the $S10$ variant as a function of hop distance under repeat/sequential/random concatenation. Across models, repeat yields zero shift, sequential yields moderate shift, and random yields the largest shift.}
  \label{fig:models_corpora_shift}
  \vspace{-20pt}
\end{wrapfigure}

\paragraph{Observed scaling differences.}
While the ranking of semantic shift is consistent, the absolute scale of the shift values can differ across models (Fig.~\ref{fig:models_corpora_shift}). This is expected, because our semantic shift metric is computed from cosine-based distances between embeddings, and different encoders induce different global angular distributions due to architectural and training choices (e.g., normalization conventions, contrastive temperature/regularization, and how aggressively the representation space is "compressed" around dominant directions).
As a result, the same underlying semantic transition in text may correspond to a larger or smaller cosine-distance change depending on the model's intrinsic geometry. Crucially, our claims in this section rely on \emph{within-model, across-pattern comparisons} (Repeat vs. Seq vs. Rand on the same corpus), where the metric is applied under a fixed encoder. Under this controlled setting, the relative ordering and monotonic alignment between shift and MPD reduction remains stable, making the conclusion robust to cross-model scale differences.

\paragraph{Implications.}
Taken together, these additional results clarify two points that are easy to conflate:
(1) baseline anisotropy (global MPD level) is model-dependent and does not by itself determine when embeddings become unreliable; and
(2) the incremental collapse induced by lengthening is strongly modulated by the degree of semantic mixing inside each unit, which is captured by semantic shift.
Therefore, semantic shift offers a more predictive lens for diagnosing when long-text embeddings will collapse (and when anisotropy is likely to translate into retrieval failures), beyond explanations that attribute collapse primarily to length alone.

\paragraph{Design implications and real-world correspondence.}
The three controlled concatenation patterns in Section~\ref{subsec:shift-vs-concentration} are not merely synthetic stress tests; they closely mirror how long "documents" arise in practice.
\textbf{Repeat concatenation} approximates long inputs with high redundancy (e.g., templated pages, boilerplate-heavy documents, repetitive logs, or duplicated passages), where length increases without introducing new semantic components.
\textbf{Sequential concatenation} resembles organically long documents (e.g. academic papers, books, or well-edited articles) in which content evolves through locally coherent discourse, introducing semantic change gradually.
In contrast, \textbf{random concatenation} serves as a proxy for heterogeneous aggregation commonly produced by real pipelines: concatenating multiple sources into a single context window (multi-page PDFs, scraped web pages with sidebars and unrelated blocks, forum threads, stitched meeting notes, or retrieval-augmented prompts that combine snippets from different topics).
These settings differ less in length than in within-unit semantic heterogeneity, precisely the factor captured by the semantic shift.

This mapping helps to clarify why "length alone" is an incomplete predictor of collapse and downstream degradation.
If length-induced embedding collapse were primarily a function of sequence length, then all long inputs of comparable length should degrade similarly.
Our results instead show that long inputs can be relatively benign when semantic shift is weak (Repeat; and sometimes Sequential on structured corpora), but can collapse severely when semantic shift is strong (Random; and Sequential on narrative texts with frequent topic transitions).
Therefore, the operational driver behind collapse in realistic workloads is often not just the token budget, but how many distinct semantic components are mixed into the same embedding unit and how fast these components evolve across the text.

From a system-design perspective, this suggests that mitigation strategies should be shift-aware rather than purely length-aware.
For example, chunking policies in retrieval or RAG pipelines are often tuned by length heuristics (fixed token windows or simple overlap).
Our findings imply that such heuristics can be suboptimal: they may unnecessarily split redundant but coherent spans (low shift) while failing to separate heterogeneous spans (high shift) that are most likely to collapse.
Instead, semantic-shift signals can be used to identify semantic boundary points where topic transitions accelerate, which are precisely the locations where aggregation is most harmful.
More broadly, semantic shift provides a principled diagnostic for long-text embedding reliability:
documents with a high shift within the unit should be decomposed, indexed, or retrieved at finer granularity, whereas low-shift documents can tolerate larger units without substantial collapse.
This perspective also reconciles why embeddings of identical length can exhibit widely different concentration behaviors (Fig.~\ref{fig:models_corpora_mpd}) and why anisotropy is not uniformly harmful across settings: it becomes most damaging when it is induced by strong semantic shift rather than by lengthening.

\vspace{8pt}
\section{Additional Retrieval Results Across Models and Corpora}
\label{sec:appendix_overlap_all_models}
\vspace{8pt}

This appendix extends Sec.~\ref{subsec:shift-vs-anisotropy} by reporting self-overlap@k results
for three embedding models (bge-large, e5-large, all-mpnet)
on two corpora (ArXiv and Alice's Adventures in Wonderland).
The main text shows only bge-large on ArXiv for brevity.
Here, we demonstrate that the key conclusion is invariant across models and corpora:
\emph{anisotropy becomes harmful primarily when induced by strong semantic shift (sequential/random mixing),
whereas anisotropy caused by lengthening (repeat) has limited impact on retrieval robustness.}

\paragraph{Figures.}
Figures~\ref{fig:bge_large_overlap}, \ref{fig:e5_large_overlap} and \ref{fig:all_mpnet_overlap} summarize the complete set of results.
Each subfigure reports average self-overlap@k ($k\in\{1,3,5\}$) between the retrieval of the original corpus $S$
and its concatenated variants ($S2$, $S5$, $S10$) in repeat/sequential/random patterns.

\begin{wrapfigure}{r}{0.5\textwidth}
  \centering
  \vspace{0pt}
  \includegraphics[width=0.5\textwidth]{bge_large_arxiv_overlap.pdf}
  \includegraphics[width=0.5\textwidth]{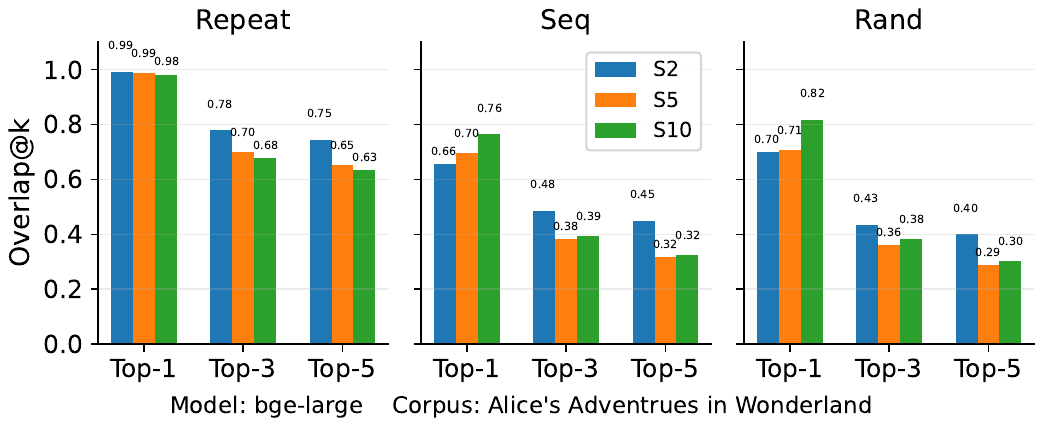}
  \vspace{-10pt}
  \caption{(Model: bge-large) Average self-overlap@k between retrieval results on the original corpus $S$ and its concatenated variants ($S2$, $S5$, $S10$) under repeat, sequential, and random patterns. Higher overlap indicates stronger semantic preservation and less retrieval damage.}
  \label{fig:bge_large_overlap}
  \vspace{-15pt}
\end{wrapfigure}

\paragraph{(1) Repeat concatenation: benign anisotropy with minimal retrieval damage.}
Across all three models and both corpora, the repeat pattern consistently yields the highest overlap and remains stable as we move from $S2$ to $S10$.
In particular, Overlap@1 is essentially preserved (typically $\approx 0.98$--$1.00$) in all settings,
and Overlap@3/5 also stays relatively high (often $\approx 0.7$--$0.8$).
This supports the main-text claim:
\emph{anisotropy/concentration induced mainly by lengthening (without semantic diversification) tends to preserve relative neighborhoods
and thus has limited impact on retrieval robustness.}

\paragraph{(2) Sequential concatenation: retrieval degrades as the semantic window expands.}
Under sequential concatenation, overlap decreases substantially and typically worsens with longer windows ($S2 \rightarrow S10$), especially for larger $k$.
On ArXiv, the trend is clear across models:
for bge-large, Overlap@1 drops to roughly $0.72 \rightarrow 0.42$ from $S2$ to $S10$, and Overlap@5 drops to roughly $0.50 \rightarrow 0.28$ (Fig.~\ref{fig:bge_large_overlap}); for e5-large, the degradation is even sharper on larger-$k$ neighborhoods (e.g., Overlap@5 around $0.55 \rightarrow 0.21$; Fig.~\ref{fig:e5_large_overlap}); and all-mpnet exhibits a similar monotone deterioration (Fig.~\ref{fig:all_mpnet_overlap}).
On Alice, sequential concatenation remains harmful across all models as well
(Figs.~\ref{fig:bge_large_overlap},~\ref{fig:e5_large_overlap},~\ref{fig:all_mpnet_overlap}), consistent with the fact that narrative discourse can accumulate semantic change even within locally adjacent spans, so enlarging the sequential window injects increasing within-unit semantic variation.

\begin{wrapfigure}{r}{0.5\textwidth}
  \centering
  \vspace{0pt}
  \includegraphics[width=0.5\textwidth]{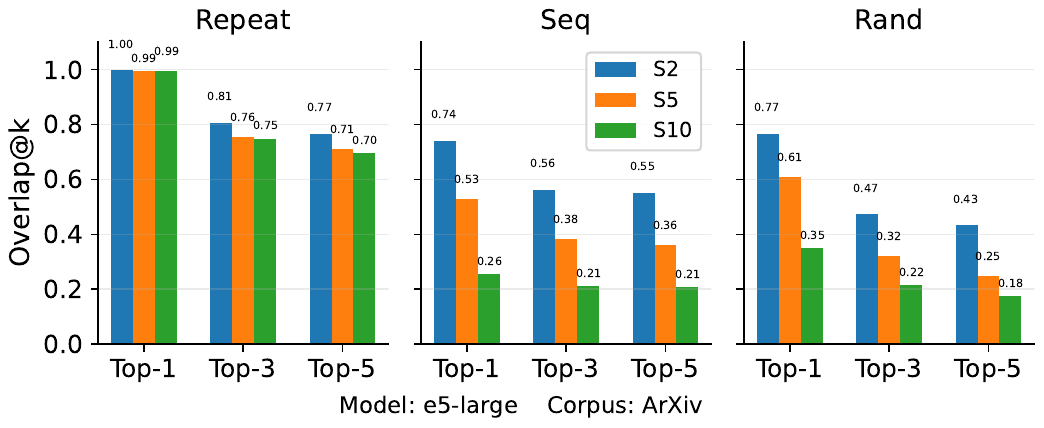}
  \includegraphics[width=0.5\textwidth]{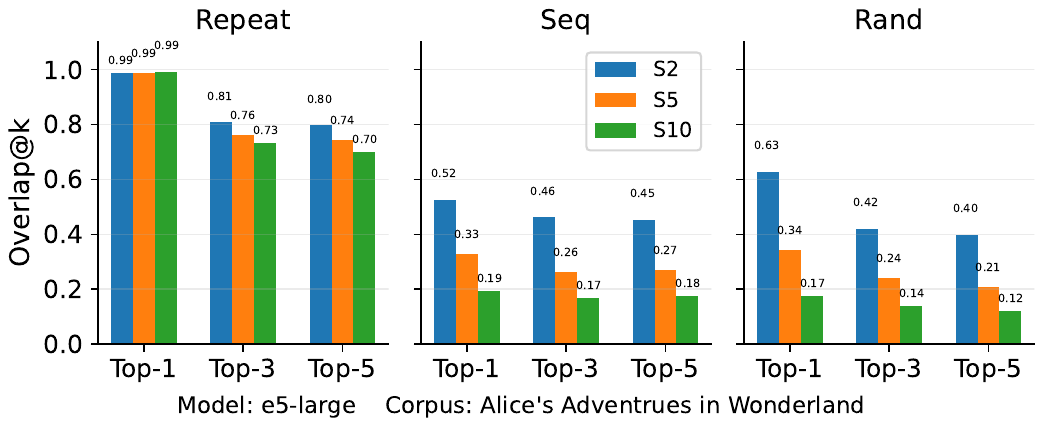}
  \vspace{-10pt}
  \caption{(Model: e5-large) Average self-overlap@k between retrieval results on the original corpus $S$ and its concatenated variants ($S2$, $S5$, $S10$) under repeat, sequential, and random patterns. Higher overlap indicates stronger semantic preservation and less retrieval damage.}
  \label{fig:e5_large_overlap}
  \vspace{-15pt}
\end{wrapfigure}

\paragraph{(3) Random concatenation: strongest retrieval damage.}
Random concatenation consistently yields the lowest overlap@k and the fastest degradation with window size. On ArXiv, bge-large shows Overlap@1 decreasing from about $0.72$ (S2) to $\approx 0.57$ (S10), and Overlap@3/5 similarly collapsing (Fig.~\ref{fig:bge_large_overlap}); e5-large shows substantial drops especially for larger $k$ (e.g., Overlap@5 reaching $\approx 0.18$ on S10; Fig.~\ref{fig:e5_large_overlap}); and all-mpnet follows the same pattern (Fig.~\ref{fig:all_mpnet_overlap}). On Alice, random remains highly damaging across models, with \texttt{e5-large} showing particularly low overlap for larger-$k$ neighborhoods (Fig.~\ref{fig:e5_large_overlap}). Overall, random mixing, which maximizes within-unit semantic heterogeneity, produces the most severe loss of neighborhood preservation, aligning with our thesis that retrieval failure is driven by \emph{semantic shift} rather than concentration alone.

\paragraph{Cross-model invariance.}
Although models differ in baseline anisotropy (e.g., e5-large is often more concentrated globally than all-mpnet),  Figures~\ref{fig:bge_large_overlap}, \ref{fig:e5_large_overlap}, and \ref{fig:all_mpnet_overlap} show stronger and more general regularity. \textbf{Across all three models, the ordering of retrieval robustness is consistent:}
\[
\text{Repeat} > \text{Sequential} > \text{Random}.
\]
That is, even when a model starts from a more anisotropic embedding space, what determines retrieval degradation under lengthening is how semantic content is mixed within each embedded unit. This invariance mirrors our concentration results (MPD/shift) and supports semantic shift as a model-agnostic driver of when anisotropy becomes harmful.

\begin{wrapfigure}{r}{0.5\textwidth}
  \centering
  \vspace{0pt}
  \includegraphics[width=0.5\textwidth]{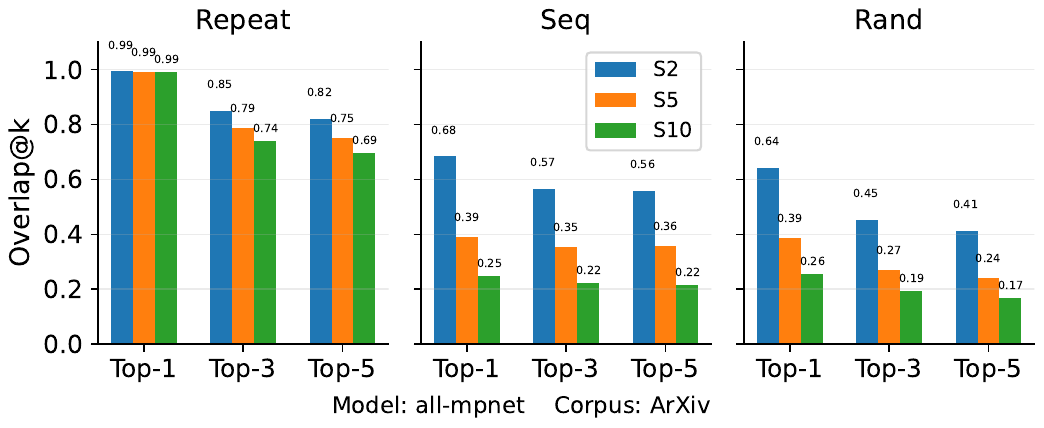}
  \includegraphics[width=0.5\textwidth]{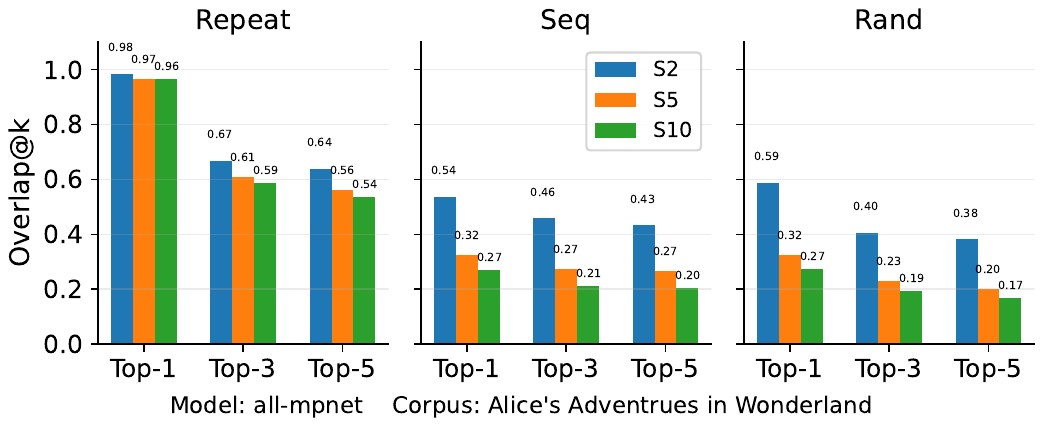}
  \vspace{-10pt}
  \caption{(Model: all-mpnet) Average self-overlap@k between retrieval results on the original corpus $S$ and its concatenated variants ($S2$, $S5$, $S10$) under repeat, sequential, and random patterns. Higher overlap indicates stronger semantic preservation and less retrieval damage.}
  \label{fig:all_mpnet_overlap}
  \vspace{-10pt}
\end{wrapfigure}

Across all embedding models and all corpora, the retrieval experiments consistently support:
(i) length-only induced concentration (repeat) is largely benign for retrieval robustness; (ii) shift-inducing transformations (sequential/random) substantially disrupt nearest-neighbor rankings; and (iii) the strength of retrieval degradation correlates with how much semantic shift is injected within each unit, providing a principled explanation for why anisotropy is not uniformly harmful and when it becomes problematic.

\section{Semantic Shift Splitter: From Analysis to a Practical Segmenter}
\label{app:ss_splitter}

\subsection{Motivation: Turning Semantic Shift into a Segmentation Signal}
\label{app:ss_splitter:motivation}
The main paper characterizes the semantic shift as a systematic shift of embedding representations as text grows, driven jointly by (i) local semantic transitions between adjacent units and (ii) the global dispersion among all units within the same context.
Beyond serving as a diagnostic lens for embedding-based retrieval, this shift signal can be directly operationalized for sentence-level segmentation.

Chunking is a core primitive in retrieval-augmented generation (RAG): an effective chunker should (a) place boundaries near meaningful topic/section transitions and (b) produce chunks with controllable granularity and stable size, since excessive size variance can lead to a mixture of tiny fragments (weak evidence) and oversized passages (token-inefficient and harder to rank).
These requirements motivate a Semantic Shift Splitter, which forms a chunk online and cuts precisely when the accumulated shift of the current segment indicates that continuing would create a semantically unstable (or internally dispersed) chunk.

\subsection{Principle: Semantic Shift within a Candidate Segment}
\label{app:ss_splitter:principle}
We segment a document at the sentence level.
Let a document be a sequence of sentences $\{s_1,\dots,s_n\}$ and let $e_i$ denote the embedding of $s_i$ (we use bge-large in all experiments).
For a candidate segment containing $k$ ordered sentences, we reuse the semantic-shift definition from Definition~\ref{def:semantic-shift}:
\[
\mathrm{Shift}(k) \;=\; \mathrm{Local}(k)\cdot \mathrm{Disp}(k).
\]
Here, $\mathrm{Local}(k)$ measures ordered stepwise shift across adjacent sentences, while $\mathrm{Disp}(k)$ measures global semantic spread among all sentences within the segment.
Their product amplifies segments that are simultaneously (i) shifting along the reading order and (ii) internally dispersed, matching the instability patterns highlighted in the main paper.
This makes $\mathrm{Shift}(\cdot)$ a natural boundary signal: when adding the next sentence sharply increases shift, the current segment is likely crossing a semantic transition.

\subsection{Algorithm: Shift-Aware Online Chunking with Adaptive Threshold}
\label{app:ss_splitter:algorithm}
The splitter constructs chunks left-to-right.
Starting from an empty chunk, it appends sentences one by one.
Before appending sentence $s_i$, it evaluates the hypothetical shift $\mathrm{Shift}(|C|+1)$; if this value exceeds a threshold $\tau$, it cuts before appending and starts a new chunk at $s_i$.
We also enforce a hard token cap to avoid overly long chunks, which is essential in RAG. See Algorithm~\ref{alg:ss_splitter} for algorithm details.

\paragraph{Adaptive threshold estimation.}
We estimate $\tau$ per document rather than fixing it globally.
For each position $t$, we compute the shift of a local window of embeddings with radius $b$, producing window-shift values $\{\hat{s}_t\}_{t=1}^{n}$, and set
\[
\tau \;=\; \mathrm{Percentile}\bigl(\{\hat{s}_t\}_{t=1}^{n},\, p\bigr),
\]
where $p$ (shift\_percentile) controls how aggressively we cut: smaller $p$ yields a smaller $\tau$ and thus more boundaries.
This document-adaptive threshold makes the splitter robust to differences in writing style, length, and topical density.

\paragraph{Efficiency.}
A naive $\mathrm{Disp}(k)$ computation is $O(k^2)$, but the online construction admits incremental updates: when adding a new sentence, we only compute similarities between the new embedding and the embeddings already in the current chunk, yielding $O(k)$ per step.
In practice, $k$ is bounded by both the shift threshold and the token cap, which makes the method efficient for typical chunk sizes.

\begin{algorithm}[t]
\caption{Semantic Shift Splitter}
\label{alg:ss_splitter}
\begin{algorithmic}[1]
\REQUIRE sentences $\{s_i\}_{i=1}^n$, embeddings $\{e_i\}_{i=1}^n$, percentile $p$, token cap $T$, min sentences per chunk $m$
\STATE Estimate $\tau$ by window shifts: $\tau \leftarrow \mathrm{Percentile}(\{\hat{s}_t\}_{t=1}^n,p)$
\STATE Initialize empty current chunk $C \leftarrow [\;]$ and state for $\mathrm{Local},\mathrm{Disp}$
\FOR{$i=1$ to $n$}
    \IF{$|C|\ge 1$ \AND tokens$(C)+$tokens$(s_i) > T$}
        \STATE output $C$; reset $C \leftarrow [\;]$
    \ENDIF
    \IF{$|C|\ge m$}
        \STATE compute hypothetical $\mathrm{Shift}(|C|+1)$ if appending $s_i$
        \IF{$\mathrm{Shift}(|C|+1) > \tau$}
            \STATE output $C$; reset $C \leftarrow [\;]$
        \ENDIF
    \ENDIF
    \STATE append $s_i$ into $C$ and update state
\ENDFOR
\IF{$C \neq [\;]$} \STATE output $C$ \ENDIF
\end{algorithmic}
\end{algorithm}


\subsection{Experimental Setup and Fair Comparison Protocol}
\label{app:ss_splitter:exp}

\begin{table*}[t]
\centering
\small
\setlength{\tabcolsep}{5pt}
\caption{ArXiv paragraph-based segmentation: comparison of Fixed Splitter, Semantic Splitter, and Semantic Shift Splitters under matched granularity ($\approx$3/5/7 sentences per chunk). Higher is better for P/R/F1; lower is better for Pk and WindowDiff. Chunk statistics (avg\_sents/chunk, var\_sents/chunk) are reported in the last two columns.}
\begin{tabular}{l c ccccc cc}
\hline
\textbf{Granularity} & \textbf{Splitter} & \textbf{P} & \textbf{R} & \textbf{F1} & \textbf{Pk}$\downarrow$ & \textbf{WD}$\downarrow$ & \textbf{avg\_sents} & \textbf{var\_sents} \\
\hline
\multirow{3}{*}{\textbf{$\approx$3}}
 & Fixed   & 0.1998 & 0.3322 & 0.2495 & 0.5353 & 0.5410 & 2.998 & 0.004 \\
 & Semantic& 0.1684 & 0.2688 & 0.2071 & 0.4088 & 0.4533 & 3.123 & 5.395 \\
 & Shift(Ours)   & \textbf{0.3809} & \textbf{0.6244} & \textbf{0.4731} & \textbf{0.3733} & \textbf{0.3894} & 3.041 & 0.977 \\
\hline
\multirow{3}{*}{\textbf{$\approx$5}}
 & Fixed   & 0.2010 & 0.2003 & 0.2007 & 0.4854 & 0.4884 & 4.998 & 0.002 \\
 & Semantic& 0.1555 & 0.1553 & 0.1554 & 0.3914 & 0.4074 & 4.990 & 16.371 \\
 & Shift(Ours)   & \textbf{0.3452} & \textbf{0.3406} & \textbf{0.3429} & \textbf{0.3763} & \textbf{0.3827} & 5.049 & 1.489 \\
\hline
\multirow{3}{*}{\textbf{$\approx$7}}
 & Fixed   & 0.2254 & 0.1603 & 0.1873 & 0.4352 & 0.4382 & 7.000 & 0.000 \\
 & Semantic& 0.1384 & 0.0968 & 0.1139 & \textbf{0.3944} & \textbf{0.4007} & 7.117 & 41.046 \\
 & Shift(Ours)   & \textbf{0.3014} & \textbf{0.2104} & \textbf{0.2478} & 0.4051 & 0.4104 & 7.134 & 1.600 \\
\hline
\end{tabular}
\label{tab:arxiv_splitter_results}
\end{table*}

\begin{table*}[t]
\centering
\small
\setlength{\tabcolsep}{5pt}
\caption{MINE paragraph-based segmentation: comparison of Fixed Splitter, Semantic Splitter, and Semantic Shift Splitters under matched granularity ($\approx$3/5/7 sentences per chunk). Higher is better for P/R/F1; lower is better for Pk and WindowDiff. Chunk statistics (avg\_sents/chunk, var\_sents/chunk) are reported in the last two columns.}
\begin{tabular}{l c ccccc cc}
\hline
\textbf{Granularity} & \textbf{Splitter} & \textbf{P} & \textbf{R} & \textbf{F1} & \textbf{Pk}$\downarrow$ & \textbf{WD}$\downarrow$ & \textbf{avg\_sents} & \textbf{var\_sents} \\
\hline
\multirow{3}{*}{\textbf{$\approx$3}}
 & Fixed   & 0.2845 & 0.3077 & 0.2956 & 0.4772 & 0.4779 & 3.000 & 0.000 \\
 & Semantic& 0.3304 & 0.3543 & 0.3420 & 0.4679 & 0.5501 & 3.026 & 7.287 \\
 & Shift(Ours)   & \textbf{0.4203} & \textbf{0.4487} & \textbf{0.4340} & \textbf{0.3907} & \textbf{0.4014} & 3.039 & 1.078 \\
\hline
\multirow{3}{*}{\textbf{$\approx$5}}
 & Fixed   & 0.2962 & 0.1923 & 0.2332 & 0.5246 & 0.5246 & 4.995 & 0.016 \\
 & Semantic& 0.3070 & 0.1993 & 0.2417 & 0.5160 & 0.5494 & 4.995 & 22.941 \\
 & Shift(Ours)   & \textbf{0.3485} & \textbf{0.2238} & \textbf{0.2725} & \textbf{0.5124} & \textbf{0.5135} & 5.049 & 2.304 \\
\hline
\multirow{3}{*}{\textbf{$\approx$7}}
 & Fixed   & 0.3417 & 0.1585 & 0.2166 & 0.5343 & 0.5346 & 6.985 & 0.090 \\
 & Semantic& 0.3026 & 0.1375 & 0.1891 & 0.5408 & 0.5573 & 7.128 & 43.186 \\
 & Shift(Ours)   & \textbf{0.3753} & \textbf{0.1737} & \textbf{0.2375} & \textbf{0.5318} & \textbf{0.5329} & 7.003 & 3.269 \\
\hline
\end{tabular}
\label{tab:mine_splitter_results}
\end{table*}

\paragraph{Datasets.}
We evaluate on two paragraph-annotated sources that reflect different discourse structures and segmentation cues.

\textbf{(1) ArXiv Abstracts.}
We use scientific abstracts from ArXiv~\citep{arxiv_abstracts_hf}, where ground-truth boundaries are defined by natural paragraph breaks.
This setting emphasizes fine-grained rhetorical and topical transitions in compact, information-dense text (Table~\ref{tab:arxiv_splitter_results}).

\textbf{(2) MINE (KG-Gen Evaluation Essays).}
We use the essay dataset MINE\citep{mo2025kggen}.
Each instance is a short essay consisting of multiple paragraphs; we treat the paragraph boundaries as the ground-truth segmentation.
Compared with ArXiv, MINE contains more narrative/expository transitions, providing a complementary testbed for semantic chunking beyond scientific writing.

For both datasets, we segment at the sentence level and define a gold boundary whenever a paragraph break occurs between two consecutive sentences.

\paragraph{Baselines.} We compare our proposed Semantic Shift Splitter against two widely-used document tiling strategies:

(1) Fixed-Length Splitting: A standard heuristic-based baseline that partitions text into chunks of a fixed number of sentences or tokens. This method ignores the underlying semantic structure but serves as a fundamental benchmark for retrieval efficiency.

(2) Standard Semantic Splitter: A dynamic splitting strategy popularized by frameworks like LlamaIndex \cite{LlamaIndex} and LangChain \cite{LangChain}. This approach determines boundaries by calculating the cosine dissimilarity between adjacent sentence embeddings and setting breakpoints at a specific percentile of local dissimilarity scores.

(3) Semantic Shift Splitter (Ours): Our proposed method, which leverages semantic shift to achieve more contextually coherent partitions.

\paragraph{Metrics.}
We report boundary Precision/Recall/F1, $P_k$, and WindowDiff (WD), and additionally track chunk-size statistics \textbf{avg\_sents/chunk} and \textbf{var\_sents/chunk}.
The variance term is practically important for RAG, since large variance introduces unstable evidence granularity and can bias retrieval and reranking.

\paragraph{Matching chunk granularity.}
To avoid confounding segmentation quality with chunk size, we compare methods under \textbf{approximately matched avg\_sents/chunk}.
Fixed controls granularity via $k$ (sentences per chunk), while Semantic/Shift mainly use semantic\_percentile and shift\_percentile, respectively (smaller percentile $\Rightarrow$ more cuts).
In practice, we (i) set Fixed to a target $k$, (ii) sweep a small set of percentile values for Semantic and Shift, and (iii) select configurations whose avg\_sents/chunk best matches the target.
This protocol produces a fair head-to-head comparison where improvements reflect better boundary placement and segmentation consistency rather than simply generating finer chunks.

\subsection{Results and Observations}
\label{app:ss_splitter:results}

Tables~\ref{tab:arxiv_splitter_results} and~\ref{tab:mine_splitter_results} summarize results on ArXiv and MINE under three matched granularities ($\approx$3/5/7 sentences per chunk).

\paragraph{Boundary quality: Semantic Shift Splitter yields consistently higher F1 at matched granularity.}
Across both datasets, the Semantic Shift Splitter achieves the strongest boundary F1 in \emph{all} three regimes.
On ArXiv (Table~\ref{tab:arxiv_splitter_results}), Shift improves F1 substantially over Fixed and the standard Semantic Splitter at $\approx$3/5/7 (0.4731/0.3429/0.2478 vs.\ 0.2495--0.2007--0.1873 for Fixed and 0.2071--0.1554--0.1139 for Semantic).
The same pattern holds on MINE (Table~\ref{tab:mine_splitter_results}), where Shift achieves the best F1 at $\approx$3/5/7 (0.4340/0.2725/0.2375), outperforming both baselines in each regime.
Notably, these gains are typically driven by improved recall while maintaining strong precision, consistent with the intuition that shift detects when a segment becomes semantically unstable and should be cut.

\paragraph{Window metrics: Semantic Shift Splitter improves Pk/WD on ArXiv and remains competitive on MINE.}
On ArXiv, Semantic Shift Splitter achieves the best (lowest) $P_k$ and WD across the first two granularities (Table~\ref{tab:arxiv_splitter_results}), indicating not only better point-wise boundary alignment (higher F1) but also stronger global consistency under window-based evaluation.
On MINE, Shift attains the lowest $P_k$ and WD at $\approx$3, and remains close to Fixed at $\approx$5 and $\approx$7 (Table~\ref{tab:mine_splitter_results}).
Overall, Shift provides a clearer and more reliable trade-off: it improves boundary F1 consistently, while maintaining competitive window metrics across two distinct domains.

\paragraph{Chunk-size stability: Semantic Shift Splitter sharply reduces variance relative to Semantic splitting.}
A persistent observation across both datasets is that the standard Semantic Splitter produces highly uneven chunk sizes even when avg\_sents/chunk is matched.
Its variance grows rapidly with granularity (ArXiv: 5.395/16.371/41.046; MINE: 7.287/22.941/43.186), indicating a mixture of very short and very long chunks (Tables~\ref{tab:arxiv_splitter_results} and~\ref{tab:mine_splitter_results}).
By contrast, Semantic Shift Splitter maintains much lower variance (ArXiv: 0.977/1.489/1.600; MINE: 1.078/2.304/3.269), approaching the regularity of Fixed splitting while remaining content-adaptive.
This stability is practically important for RAG, as it reduces both fragmented evidence (overly short chunks) and token-inefficient passages (overly long chunks), making retrieval and reranking behavior more predictable.

Across ArXiv and MINE datasets, and under matched granularity, the Semantic Shift Splitter consistently improves boundary F1 and generally yields better or competitive $P_k$/WD, while dramatically reducing chunk-size variance compared to the standard Semantic Splitter (Tables~\ref{tab:arxiv_splitter_results} and~\ref{tab:mine_splitter_results}).
These results support the claim that explicitly combining local semantic transitions with global dispersion leads to a more accurate and controllable segmentation strategy.

\subsection{Summary and Commentary}
\label{app:ss_splitter:summary}
This appendix turns the semantic shift from an analytic quantity into a practical segmentation mechanism.
The Semantic Shift Splitter cuts when a segment’s joint local shift and global dispersion indicate semantic instability, using a document-adaptive threshold and an online construction compatible with RAG constraints.
Empirically, under matched granularity, it consistently improves boundary F1 over Fixed and Semantic splitting, while dramatically reducing chunk-size variance compared to the Semantic Splitter.
These results support the broader takeaway of the main paper: semantic shift is not only a fundamental challenge for embeddings and retrieval but also a useful, controllable signal for building more reliable text processing components.



\end{document}